\documentclass[10pt,twocolumn,letterpaper]{article}

\usepackage{wacv}              
\usepackage[accsupp]{axessibility} 
\usepackage[normalem]{ulem}
\usepackage{microtype}
\usepackage{graphicx}
\usepackage{booktabs}
\usepackage{multirow}
\usepackage{amsmath,amssymb,mathtools}
\usepackage{amsthm}
\usepackage{subcaption}
\usepackage{enumitem}
\usepackage{xcolor}
\usepackage{float}
\usepackage{pifont}
\newcommand{\cmark}{\textcolor{green!55!black}{\ding{51}}}
\newcommand{\xmark}{\textcolor{red!70!black}{\ding{55}}}

\theoremstyle{plain}
\newtheorem{theorem}{Theorem}[section]

\newtheorem{proposition}{Proposition}[section]
\newtheorem{corollary}[theorem]{Corollary}
\theoremstyle{definition}
\newtheorem{definition}[theorem]{Definition}
\theoremstyle{remark}
\newtheorem{remark}[theorem]{Remark}

\definecolor{wacvblue}{rgb}{0.21,0.49,0.74}
\usepackage[pagebackref,breaklinks,colorlinks,allcolors=wacvblue]{hyperref}

\usepackage[capitalize,noabbrev]{cleveref}
\crefname{proposition}{Prop.}{Props.}
\Crefname{proposition}{Proposition}{Propositions}
\crefname{theorem}{Thm.}{Thms.}
\Crefname{theorem}{Theorem}{Theorems}
\crefname{corollary}{Cor.}{Cors.}
\Crefname{corollary}{Corollary}{Corollaries}
\crefname{definition}{Def.}{Defs.}
\Crefname{definition}{Definition}{Definitions}

\def\wacvPaperID{600} 
\def\confName{WACV}
\def\confYear{2027}

\title{AI-Generated Image Detectors Exploit Global Artifacts:\\
       Evidence from Inpainting Exchange}

\author{Elif Nebioglu\thanks{Equal contribution.}\\
LTCI, T\'el\'ecom Paris\\
Institut Polytechnique de Paris, France\\
{\tt\small elif.nebioglu@ip-paris.fr}
\and
Emirhan Bilgi\c{c}\footnotemark[1]\\
Department of Computing\\
Imperial College London, UK\\
{\tt\small e.bilgic26@imperial.ac.uk}
\and
Adrian Popescu\\
Universit\'e Paris-Saclay, CEA, List\\
F-91120, Palaiseau, France\\
{\tt\small adrian.popescu@cea.fr}
}

\begin{document}
\maketitle

\begin{abstract}
Modern diffusion-based inpainting enables realistic, localized image manipulation, raising critical challenges for reliable detection. We show that state-of-the-art AI-generated image detectors largely fail to identify \emph{what} was synthesized and instead exploit a \emph{global} artifact introduced as a side effect of the generation pipeline: latent-diffusion inpainting passes the entire image through a VAE encoder--decoder, inducing a subtle but pervasive spectral shift even in unedited regions. We turn this observation into a simple, model-free diagnostic, \emph{Inpainting Exchange} (INP-X), that restores the original pixels outside the edited region while preserving all synthesized content. To study the effect at scale, we build a benchmark of $30$K matched real, standard-inpainted, and exchanged triplets ($90$K images) spanning four datasets and three inpainters. INP-X requires no model access and is a single compositing step, yet it collapses pretrained detectors, commercial APIs, and bias-aware methods toward chance (commercial APIs: $>$$90\%\!\to\!\sim$$55\%$; the bias-aware DDA: $97.6\%\!\to\!58.4\%$). We give formal propositions linking the effect to high-frequency attenuation from the VAE information bottleneck, and we causally isolate it with a factorial battery and a variance decomposition that attributes $97\%$ of the background spectral shift to the autoencoder round-trip alone. A perceptual study ($n{=}153$) shows the shift is also perceptible to humans, whose accuracy falls from $63.8\%$ to $56.2\%$ ($p{=}0.002$): significantly reduced, though not eliminated. Finally, we show that mechanism-driven mitigations (mixed-distribution and localization-supervised training) improve robustness, but detecting INP-X edits remains an open problem. Our benchmark and code are available at: \url{https://github.com/emirhanbilgic/INP-X}
\end{abstract}
\section{Introduction}

Inpainting, filling a masked region with plausible content, has become a mainstream editing capability through diffusion-based tools~\cite{rombach2022high,podell2024sdxl} and commercial services~\cite{deepimage2024}. This accessibility poses real risks for misinformation and content authenticity~\cite{verdoliva2020media}, making reliable detection of inpainted content critical. Both academic methods~\cite{corvi2023diffusion,guillaro2025biasfree,chen2025dual,ojha2023universal} and commercial APIs~\cite{sightengine2024,hivemoderation2024} report high accuracy ($>$$90\%$) on \emph{vanilla} inpainting/fully AI-generated images, where the generator's raw output is presented for detection unmodified.

A fundamental question remains, however: \emph{what are these detectors actually learning?} For inpainting, which synthesizes only a local region, one would expect a detector to focus on the generated content inside the edited region. Instead, we show that most state-of-the-art detectors rely primarily on a \emph{global artifact} introduced by the generation pipeline, a form of shortcut learning~\cite{geirhos2020shortcut} that undermines their intended purpose. Latent-diffusion inpainting processes the \emph{entire} image through a VAE encoder--decoder, so even perfectly preserved background latents are re-rendered by the decoder, leaving a subtle yet pervasive spectral shift across the whole image~\cite{frank2020frequency,corvi2023diffusion}. This global fingerprint is a trivial signal that bypasses the need to identify locally generated content~\cite{rajan2025staypositive,yan2025sanity}.

\begin{figure}[t!]
    \centering
    \includegraphics[width=0.98\columnwidth]{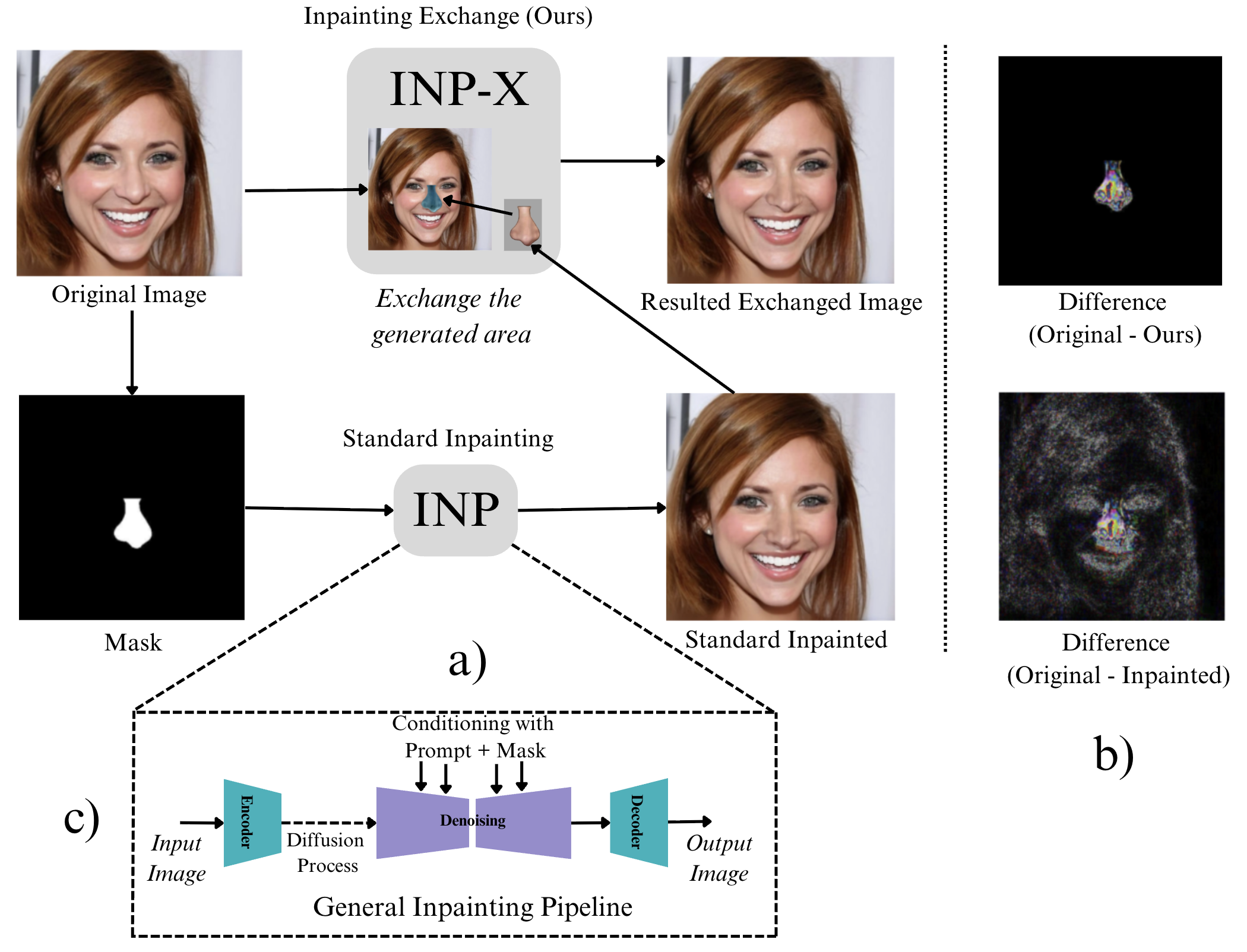}
    \caption{\textbf{Inpainting Exchange (INP-X).} (a) The unedited regions of a diffusion-inpainted image are replaced with the corresponding original pixels, preserving the synthesized masked content while removing the global VAE artifact. (b) Difference maps: standard inpainting (Original$-$Inpainted) differs across the \emph{whole} image, whereas INP-X (Original$-$Ours) differs only inside the mask. (c) The latent-diffusion inpainting pipeline (INP) that produces the global artifact.}
    \label{fig_teaser}
\end{figure}

We turn this diagnosis into a deliberately simple operation, \emph{Inpainting Exchange} (INP-X, \cref{fig_teaser}): after inpainting, we restore the original pixels wherever the mask is $0$ and keep the synthesized pixels wherever it is $1$. If a detector truly identifies synthetic content, it should still flag exchanged images, since the fake content is untouched. It does not: exchange removes the global artifact and detection collapses toward chance. The simplicity is the point. INP-X is a single compositing step, requires no access to the detector or the generator, and yet takes two commercial APIs from over $90\%$ to $\sim$$55\%$ accuracy. This is not a niche construction: background restoration is already a one-flag option in widely deployed tools (\eg, diffusers \texttt{padding\_mask\_crop}, Automatic1111 ``Only masked,'' Photoshop Generative Fill), all of which emit near-INP-X distributions with no intervention from us (\cref{sec:supp_deployed}).

Crucially, we do not merely observe the collapse; we explain and causally attribute it. We prove that an $L_2$-trained VAE attenuates high-frequency, non-semantic content (\cref{prp_attenuation}) and that exact background restoration provably reduces the distributional divergence any image-wide detector can exploit (\cref{thm_divergence}). We then isolate the cause empirically with a factorial battery that toggles generated content, VAE-processed background, and compositing seam independently: a \emph{real} photo passed through the autoencoder alone is flagged fake at $0.91$, while swapping every synthesized pixel for real content moves the fake rate by only $\sim$$1$ point. A variance decomposition attributes $97\%$ of the background spectral shift to the VAE round-trip. The vulnerability is the background artifact, not the synthesized content, the compositing, or a domain shift from exchange.

\noindent Our contributions are:
\begin{itemize}[leftmargin=*, itemsep=1.5pt, parsep=0pt, topsep=1pt, partopsep=0pt]
    \item \textbf{A model-free attack/diagnostic.} INP-X (\cref{fig_teaser}) is one compositing step that collapses 14 detectors and 2 commercial APIs toward chance, exposing that state-of-the-art inpainting detectors key on a global VAE artifact rather than the generated content.
    \item \textbf{Theoretically-sound explanation of the effect.} Formal propositions link the shift to high-frequency attenuation from the VAE information bottleneck (\cref{prp_attenuation}) and show exchange provably tightens an upper bound on the achievable detector accuracy (\cref{thm_divergence}, \cref{cor_accuracy_ceiling}).
    \item \textbf{Direct causal attribution.} A factorial battery and a variance decomposition ($97\%$ VAE) separate the background artifact from content, seam, and hybrid-distribution confounds, addressing whether the collapse is real detection failure or a construction side effect.
    \item \textbf{Comprehensive evaluation and a benchmark.} We assemble $30$K matched real/INP/INP-X triplets ($90$K images) over four datasets and three inpainters, and evaluate frequency, CLIP-semantic (\eg, C2P-CLIP), bias-aware, and commercial detectors. A human study ($n{=}153$) further shows that INP-X edits are substantially harder for people to flag than standard inpainting (accuracy $63.8\%\!\to\!56.2\%$, $p{=}0.002$), so the vulnerability degrades human as well as machine detection.
    \item \textbf{Mechanism-driven mitigations.} Mixed-distribution and localization-supervised training improve robustness and cross-distribution generalization (small but significant gains), yet detecting INP-X edits remains hard, indicating an architectural, not a data-budget, limitation.
\end{itemize}

Taken together, our results contribute to showing that AI-generated image detection is far from solved: robust evaluation must go beyond ``clean'' generated images and include realistic post-edits, such as INP-X, that mirror what deployed pipelines already produce.

\section{Related Work}

\paragraph{AI-generated image detection.}
Generative models leave characteristic fingerprints exploitable for detection~\cite{marra2019gans}. Frank~\etal~\cite{frank2020frequency} identified these as frequency-domain artifacts from upsampling, and Wang~\etal~\cite{wang2020cnn} showed a single-generator CNN classifier generalizes across architectures. For diffusion models, Corvi~\etal~\cite{corvi2023diffusion} found distinct spectral patterns from the denoising process, and Wang~\etal~\cite{wang2023dire} proposed the Diffusion Reconstruction Error as a detection signal. Recent work targets spectral invariants~\cite{karageorgiou2025spai} and artifact purification~\cite{meng2024artifact}. These detectors are trained and validated on \emph{fully} synthetic images; we ask how they behave when only a local region is generated.

\paragraph{Universal detectors and generalization.}
The search for detectors that generalize across generators has produced several influential approaches. Ojha~\etal~\cite{ojha2023universal} learn generator-agnostic representations in CLIP feature space, and Cozzolino~\etal\ pair CLIP embeddings with lightweight classifiers for strong few-shot transfer~\cite{cozzolino2023clip} and detect synthetic images zero-shot from a lossless coder's prediction error~\cite{cozzolino2024zeroshot}. Yet recent studies question whether CLIP-based detectors rely on artifacts or semantic shortcuts: Chu~\etal~\cite{chu2025semantics} show patch shuffling forces models toward artifact cues, and Yan~\etal~\cite{yan2025nsnet} decouple semantic information from forgery-relevant features via null-space projection. We show that even semantic detectors (\eg, C2P-CLIP~\cite{tan2025c2pclip}) collapse under INP-X.
\vspace{-1em}
\paragraph{Shortcut learning and dataset bias.}
Mounting evidence suggests detector performance often originates in dataset bias rather than genuine capability. Rajan~\etal~\cite{rajan2025staypositive} argue detectors should ignore real-image features and rely only on generator artifacts; Ha~\etal~\cite{ha2024organic} show detectors struggle as generators improve; Yan~\etal~\cite{yan2025sanity} provide a sanity check revealing that many reported gains stem from bias; and Grommelt~\etal~\cite{grommelt2024jpeg} identify systematic JPEG-compression and resolution biases. Guillaro~\etal~\cite{guillaro2025biasfree} propose B-Free, a bias-free training paradigm that pairs each real image with a content-matched reconstruction, being the strongest baseline in our study and the clearest illustration of our thesis: although it is explicitly built to shed biases and localize inpainting, it still keys on the global VAE artifact, retaining high aggregate accuracy on INP-X yet collapsing toward chance on the small-mask majority (\cref{subsec:pretrained}). Li~\etal~\cite{li2024adversarial} further highlight persistent vulnerabilities. We contribute a controlled, mechanism-grounded intervention that isolates the shortcut in the inpainting setting.
\vspace{-0.6em}

\paragraph{Localization and forensics.}
Beyond binary detection, splice-detection methods~\cite{farid2009image} analyze inconsistencies in JPEG artifacts, lighting, and noise. Photo-Response Non-Uniformity (PRNU)~\cite{lukas2006digital} verifies sensor-specific noise; our \cref{prp_attenuation} predicts that VAE reconstruction attenuates such high-frequency sensor noise, explaining why these methods succeed on standard inpainting (PRNU disrupted globally) but fail on INP-X (PRNU restored in the background). Noiseprint~\cite{Cozzolino2019_Noiseprint} learns camera-model fingerprints robust to in-camera processing. Proposing a new localization method is not in scope. Instead, we show that existing general-purpose detectors fail to localize inpainted content and instead read global VAE artifacts.
\vspace{-0.6em}
\paragraph{Evasion, augmentation, and benchmarks.}
Prior evasion methods~\cite{hou2023evading,stamm2013forensics} \emph{add} perturbations or degrade images to disrupt detection; INP-X instead \emph{removes} the global artifact by restoring original pixels: a diagnostic, not an adversarial perturbation, and not a generic augmentation, since it drives the background divergence to exactly zero (a generic denoise-and-resynthesize baseline only partially reduces detection; \cref{sec:supp_denoise}). For evaluation, GenImage~\cite{zhu2023genimage} covers full-image synthesis, while Semi-Truths~\cite{pal2024semitruths} pairs real images with inpainting outputs and ground-truth masks. We extend Semi-Truths with INP-X triplets to isolate whether detectors identify \emph{what was generated} rather than \emph{how it was processed}.

\section{Theoretical Analysis}
\label{sec:theory}

We first define diffusion inpainting, then analyze the source of the global artifact in latent diffusion models (LDMs) and prove why INP-X reduces detectability.

\subsection{Latent Diffusion Inpainting}
Let $x \in \mathbb{R}^{H \times W \times 3}$ be a real image and $M \in \{0,1\}^{H \times W}$ a binary mask, where $1$ marks the region to inpaint. LDMs operate in a compressed latent space via an encoder $\mathcal{E}$, decoder $\mathcal{D}$, and denoiser $\epsilon_\theta$. The inpainted image is
\begin{equation}
    \tilde{x} = \mathcal{D}\big(\text{ReverseDiffusion}(\mathcal{E}(x), M)\big).
\end{equation}
Even if the reverse process perfectly preserves the unmasked latents, the final image is produced by the decoder $\mathcal{D}(\cdot)$, which acts over the entire spatial domain. As observed by Hou~\etal~\cite{hou2025towards} and illustrated in \cref{fig_global_impact}, this introduces subtle yet pervasive changes in the unmasked regions, differentiating them from the original source.

\begin{figure}[t]
    \centering
    \begin{subfigure}[t]{0.45\columnwidth}
        \centering
        \includegraphics[width=\linewidth]{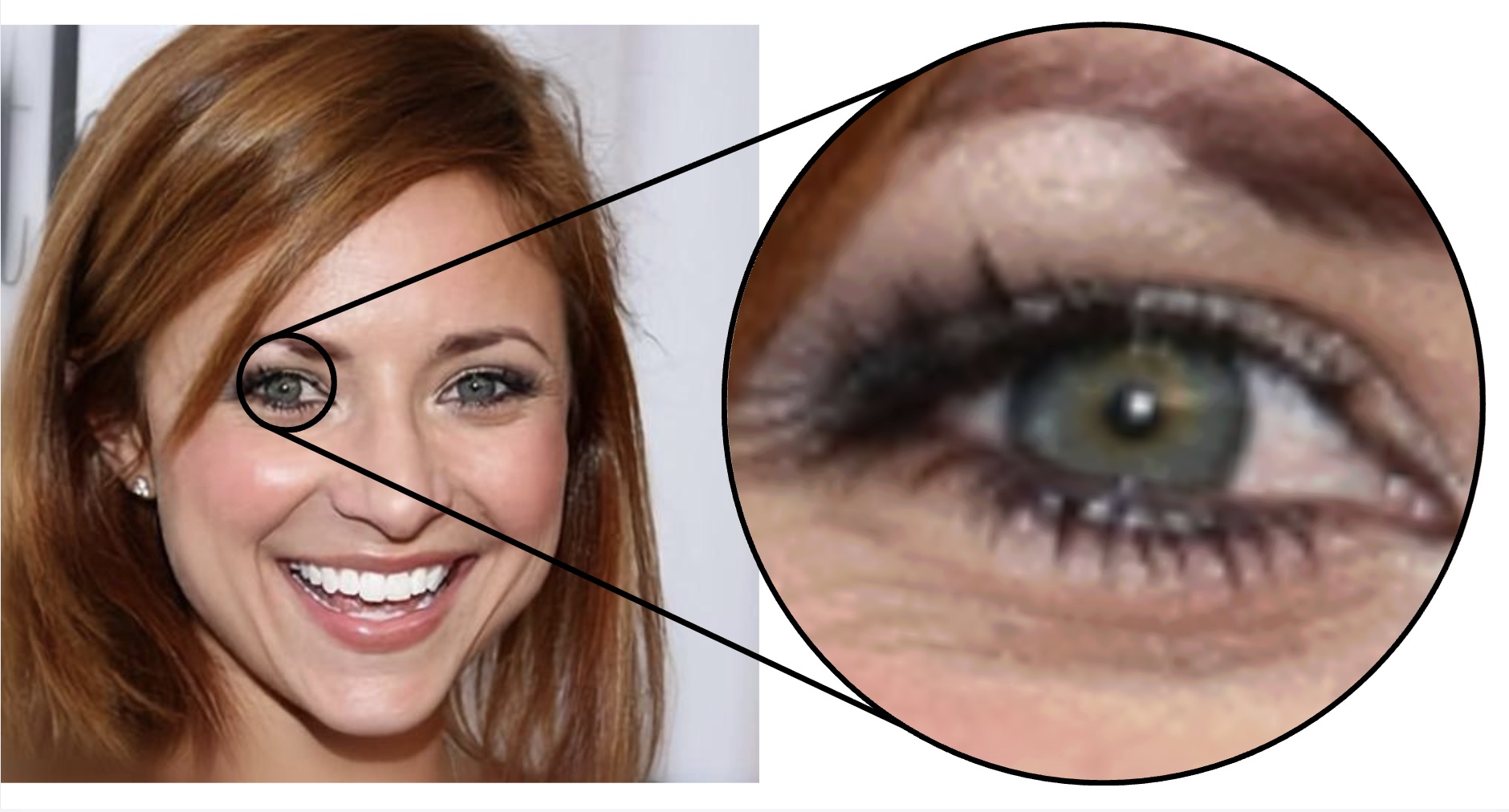}
        \caption{Original}
    \end{subfigure}
    \hfill
    \begin{subfigure}[t]{0.45\columnwidth}
        \centering
        \includegraphics[width=\linewidth]{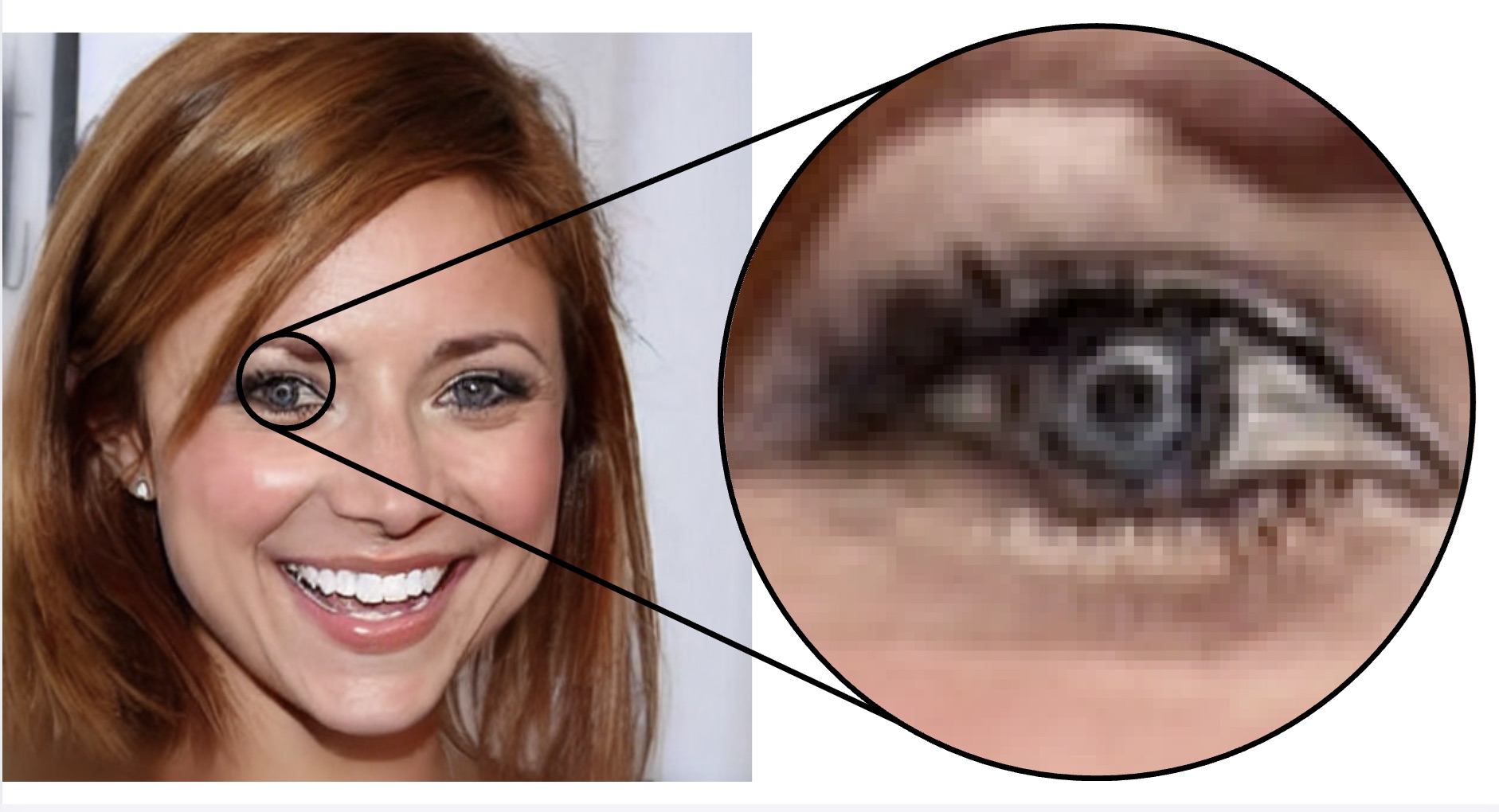}
        \caption{Standard inpainting (INP)}
    \end{subfigure}
    \caption{Impact of standard inpainting on \emph{unmasked} regions. Although the mask targets the nose, the VAE encode--decode cycle alters unmasked pixels (\eg, the highlighted distortion around the eye), distinguishing them from the original. Inpainting via Stable Diffusion v1.4.}
    \label{fig_global_impact}
\end{figure}

\subsection{Spectral Bias in VAE Reconstruction}
Standard LDMs use an autoencoder $\mathcal{T} = \mathcal{D} \circ \mathcal{E}$ mapping between pixel space $\mathcal{X}$ and latent space $\mathcal{Z}$. We model a real image as a superposition of a semantic signal $s$ and stochastic sensor noise $n$ (\eg, photon shot noise; PRNU, which is fixed per camera but random and non-semantic across the image population): $x = s + n$, with $n \sim \mathcal{N}(0, \Sigma_{noise})$.

\begin{definition}[Spectral power spectrum]
Let $\mathcal{F}[x](\omega)$ be the Fourier transform of $x$ at frequency $\omega$. Define $S_x(\omega) = \mathbb{E}\big[|\mathcal{F}[x](\omega)|^2\big]$.
\end{definition}

\begin{proposition}[Variance contraction in VAEs]
\label{prp_attenuation}
Let $\mathcal{T}$ minimize a reconstruction objective dominated by the $L_2$ norm, $\mathcal{L} = \|x - \mathcal{T}(x)\|^2$. Assume the latent $z=\mathcal{E}(x)$ captures the semantic signal $s$ but is approximately independent of the stochastic high-frequency noise $n$ due to the information bottleneck (\cref{rem:modern_vae}). Then the reconstruction $\tilde{x}=\mathcal{T}(x)$ satisfies
\begin{equation}
    S_{\tilde{x}}(\omega) \le S_x(\omega) \quad \forall\, \omega \in \Omega_{high},
\end{equation}
where $\Omega_{high}$ is the band dominated by sensor noise $n$.
\end{proposition}

\begin{proof}
The $L_2$-optimal reconstruction approximates the conditional mean, $\tilde{x} \approx \mathbb{E}[x \mid z]$. By the law of total variance, $\mathrm{Var}(x) = \mathrm{Var}(\mathbb{E}[x|z]) + \mathbb{E}[\mathrm{Var}(x|z)]$. Since $n$ is stochastic and non-semantic, it is not encoded in $z$, so its variance falls into the residual $\mathbb{E}[\mathrm{Var}(x|z)]$. Since $\mathcal{F}$ is linear, $\mathcal{F}[\tilde{x}](\omega) \approx \mathbb{E}[\mathcal{F}[x](\omega) \mid z]$, and the same decomposition applied to each Fourier coefficient gives $S_{\tilde{x}}(\omega) \le S_x(\omega)$, with the deficit carried by the power of $n$. As $n$ is high-frequency dominant, the gap is concentrated (and strict) on $\Omega_{high}$.
\end{proof}

\begin{figure*}[t]
    \centering
    \includegraphics[width=0.85\textwidth]{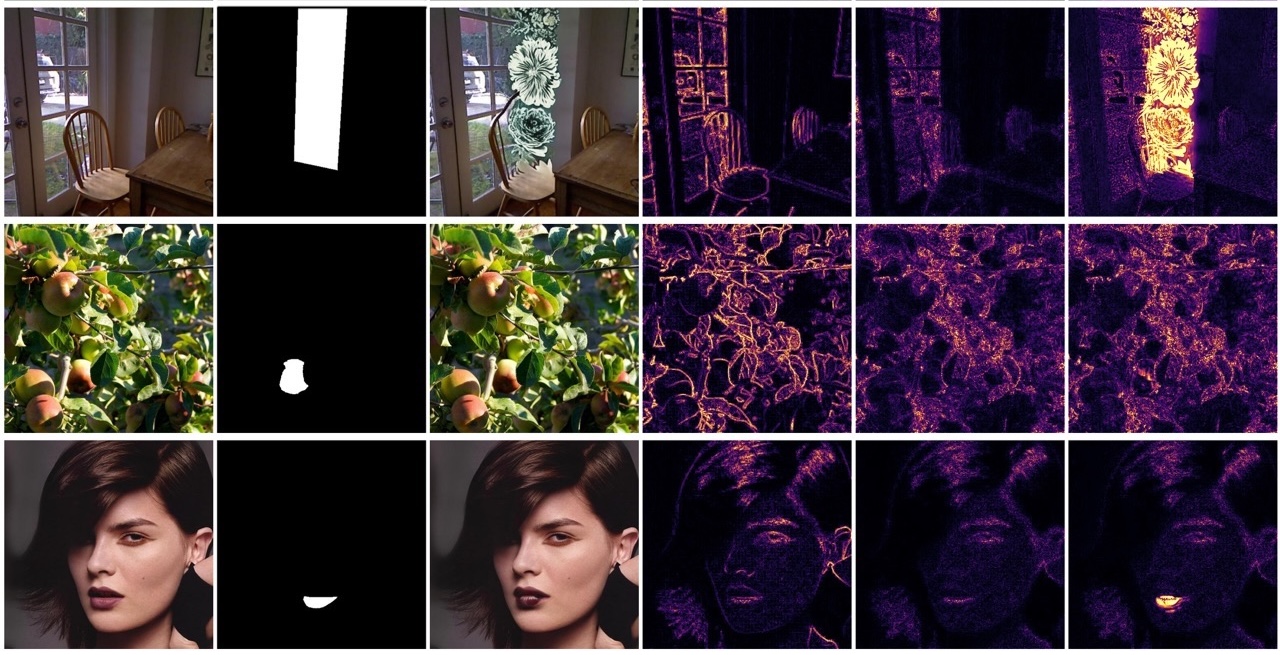}
    \caption{\textbf{Verification of VAE-induced artifacts.} Left to right: (1)~original; (2)~mask; (3)~inpainted; (4)~high-frequency content of the original (Laplacian high-pass); (5)~pure VAE reconstruction error, $|\mathcal{D}(\mathcal{E}(x))-x|$, \emph{no} diffusion; (6)~inpainting difference $|\tilde{x}-x|$. The strong structural correlation between (5) and (6) identifies the VAE as the primary source of the global artifact.}
    \label{fig_vae_verification}
\end{figure*}

\begin{remark}[Scope and modern VAEs]
\label{rem:modern_vae}
The independence assumption is empirically motivated, not guaranteed. Modern VAEs use hybrid losses (perceptual~\cite{johnson2016perceptual} + adversarial + $L_2$), yet high-frequency attenuation persists: perceptual losses operate on low-pass VGG features that ignore sensor noise, and $8{\times}$ spatial compression~\cite{rombach2022high} imposes a hard limit on encoding non-semantic high-frequency content. We confirm this on SD~1.4's VAE below and extend it to SDXL and FLUX.1 in \cref{sec:supp_vae_advanced}; a rate--distortion (wavelet) refinement is in \cref{sec:supp_wavelet}.
\end{remark}

\paragraph{Empirical validation: the VAE is the source.}
To verify that the global artifact stems from the autoencoder (as \cref{prp_attenuation} predicts), we isolate the pre-trained VAE of SD~1.4 and apply pure encode--decode to real images, \emph{without} any diffusion or inpainting. \cref{fig_vae_verification} shows a striking spatial similarity between the VAE reconstruction error and the artifact seen in standard inpainting; image-level correlations are strong (Pearson $r{=}0.941$ between VAE loss and high-frequency content, $r{=}0.600$ between VAE loss and inpainting difference; full statistics in \cref{sec:supp_correlation}). To attribute the effect quantitatively, we decompose the background spectral shift by pipeline stage (\cref{tab_variance_decomp}): the autoencoder round-trip alone reproduces $97.3\%$ of the shift, while adding the (latent-frozen) denoiser contributes only $1.6\%$. The background artifact is a decoder property, not a product of the diffusion process or of the exchange operation.

\subsection{INP-X and Divergence Reduction}
We measure detectability via the KL divergence between the real distribution $P_{real}$ and a manipulated distribution $P_{manip}$; lower divergence means harder to distinguish. Partition the image domain $\Omega$ into background $\Omega_{bg}$ ($M{=}0$) and foreground $\Omega_{fg}$ ($M{=}1$), with contents $x_{bg}$, $x_{fg}$.

\begin{proposition}[Divergence reduction via exchange]
\label{thm_divergence}
Let $P_{std}$ and $P_{ex}$ be the distributions of standard-inpainted and exchanged images. Assume standard inpainting induces a non-zero background spectral shift (\cref{prp_attenuation}) and that this background term dominates the foreground change, \ie
\begin{equation}
\label{eq:dominance}
D_{bg}(std) \;>\; \mathbb{E}[D_{fg|bg}(ex)] - \mathbb{E}[D_{fg|bg}(std)]. \tag{$\star$}
\end{equation}
Then $D_{KL}(P_{real} \| P_{ex}) < D_{KL}(P_{real} \| P_{std})$.
\end{proposition}

\begin{proof}
By the chain rule for KL divergence,
\begin{equation}
    \begin{aligned}
    D_{KL}(P \| Q) ={}& D_{KL}\big(P(x_{bg}) \| Q(x_{bg})\big) \\
   &+ \mathbb{E}_{x_{bg} \sim P}\big[D_{KL}(P(x_{fg}|x_{bg}) \| Q(x_{fg}|x_{bg}))\big].
    \end{aligned}
\end{equation}
Under standard inpainting the VAE reconstructs the background, so $D_{KL}(P_{real}(x_{bg}) \| P_{std}(x_{bg})) > 0$ by \cref{prp_attenuation}. Under INP-X the background is copied exactly ($x^{ex}_{bg}=x_{bg}$), forcing that term to $0$, so the total divergence reduces to the foreground conditional alone. Condition~(\ref{eq:dominance}) guarantees the net change is a strict decrease.
\end{proof}

\begin{remark}[On the dominance assumption]
Condition~(\ref{eq:dominance}) is not universal: it can fail when the mask covers nearly the whole image, at which point INP-X degenerates to INP. Empirically it holds for mask areas up to ${\sim}80\%$, covering $99.0\%$ of our benchmark (\cref{sec:supp_dominance}); we therefore state it as an explicit hypothesis rather than a guarantee.
\end{remark}

\begin{corollary}[Upper bound on detector accuracy]
\label{cor_accuracy_ceiling}
With equal class priors, the Bayes-optimal accuracy is $\mathrm{Acc}^* = \tfrac{1}{2} + \tfrac{1}{2}\mathrm{TV}$. By Pinsker's inequality, $\mathrm{TV}(P,Q)^2 \le \tfrac{1}{2}D_{KL}(P\|Q)$. With \cref{thm_divergence},
\begin{equation}
\begin{split}
    \mathrm{Acc}^*_{ex} &= \tfrac{1}{2} + \tfrac{1}{2}\mathrm{TV}(P_{real},P_{ex})
    \le \tfrac{1}{2} + \sqrt{\tfrac{1}{8}D_{KL}(P_{real}\|P_{ex})} \\
    &< \tfrac{1}{2} + \sqrt{\tfrac{1}{8}D_{KL}(P_{real}\|P_{std})},
\end{split}
\end{equation}
where the last term is the corresponding Pinsker bound for standard inpainting, so INP-X provably tightens an upper bound on the achievable accuracy of any detector.
\end{corollary}

\section{Experiments}
\label{sec:experiments}

\subsection{Benchmark and setup}
\label{subsec:setup}
We extend Semi-Truths~\cite{pal2024semitruths} into a benchmark of \textbf{$30{,}000$ matched real/INP/INP-X triplets ($90{,}000$ images)} spanning CelebA-HQ~\cite{karras2018progressive}, CityScapes~\cite{Cordts2016Cityscapes}, OpenImages~\cite{kuznetsova2020open}, and SUN-RGBD~\cite{song2015sun} (triplet counts $2{,}903 / 12{,}041 / 7{,}112 / 7{,}944$). Every image carries a ground-truth mask, enabling controlled analysis. Inpainting uses Kandinsky~2.2~\cite{razzhigaev2023kandinsky}, OpenJourney~\cite{openjourney2023}, and SD~v1.4~\cite{rombach2022high}; masks and prompts are inherited from Semi-Truths and SAM-refined. The mask-size distribution is realistic and small-mask-dominated ($76\%$ of masks cover $\le 20\%$ of the image; median ${\approx}4\%$), matching Semi-Truths ($76.5\%$) and OABench ($70.2\%$). To remove format shortcuts~\cite{grommelt2024jpeg,zhu2023genimage}, all arms share a byte-identical JPEG save/load path (\cref{sec:supp_jpeg}). We evaluate 13 pretrained detectors, a semantic C2P-CLIP detector~\cite{tan2025c2pclip}, and two commercial APIs\footnote{Commercial APIs evaluated on $1{,}000$ images each, due to rate limits.} in a binary setting, then fine-tune four architectures on INP/INP-X for classification and localization.

\subsection{Pretrained detectors collapse under INP-X}
\label{subsec:pretrained}

\begin{table}[t]
\centering
\caption{\textbf{Pretrained detectors and commercial APIs}: real vs.\ standard inpainting (INP) and Inpainting Exchange (INP-X). Every family collapses toward chance once the global artifact is removed; the semantic detector C2P-CLIP is not exempt. UnivFD is already at chance on INP (floor effect), not robust. Full detector results in \cref{tab_all_detectors_full}.$^\dagger$}
\label{tab_all_detectors}
\vskip 0.05in
\resizebox{\columnwidth}{!}{%
\begin{sc}
\begin{tabular}{llccccc}
\toprule
Detector & Data & Acc & AUC & Prec & Rec & F1 \\
\midrule
\multirow{2}{*}{\shortstack[l]{SPAI\,\tiny\cite{karageorgiou2025spai}}}
& INP & 0.661 & 0.743 & 0.638 & 0.744 & 0.687 \\
& \textbf{INP-X} & \textbf{0.542} & \textbf{0.567} & \textbf{0.546} & \textbf{0.506} & \textbf{0.525} \\
\midrule
\multirow{2}{*}{\shortstack[l]{Corvi2023\,\tiny\cite{corvi2023diffusion}}}
& INP & 0.942 & 0.989 & 0.995 & 0.890 & 0.939 \\
& \textbf{INP-X} & \textbf{0.554} & \textbf{0.519} & \textbf{0.959} & \textbf{0.114} & \textbf{0.203} \\
\midrule
\multirow{2}{*}{\shortstack[l]{C2P-CLIP\,\tiny\cite{tan2025c2pclip}}}
& INP & 0.696 & 0.780 & 0.712 & 0.660 & 0.685 \\
& \textbf{INP-X} & \textbf{0.562} & \textbf{0.604} & \textbf{0.598} & \textbf{0.381} & \textbf{0.465} \\
\midrule
\multirow{2}{*}{\shortstack[l]{UnivFD\,\tiny\cite{ojha2023universal}}}
& INP & 0.552 & 0.632 & 0.746 & 0.159 & 0.262 \\
& INP-X & 0.550 & 0.592 & 0.739 & 0.154 & 0.254 \\
\midrule
\multicolumn{7}{l}{\textit{State-of-the-art bias-aware methods}} \\
\midrule
\multirow{2}{*}{\shortstack[l]{DDA\,\tiny\cite{chen2025dual}}}
& INP & 0.976 & 0.998 & 0.974 & 0.979 & 0.976 \\
& \textbf{INP-X} & \textbf{0.584} & \textbf{0.616} & \textbf{0.882} & \textbf{0.195} & \textbf{0.319} \\
\midrule
\multirow{2}{*}{\shortstack[l]{B-Free\,\tiny\cite{guillaro2025biasfree}}}
& INP & 0.974 & 0.996 & 0.973 & 0.975 & 0.974 \\
& \textbf{INP-X} & \textbf{0.811} & \textbf{0.873} & \textbf{0.960} & \textbf{0.649} & \textbf{0.774} \\
\midrule
\multicolumn{7}{l}{\textit{Commercial APIs}} \\
\midrule
\multirow{2}{*}{Hive Moderation}
& INP & 0.914 & 0.921 & 0.993 & 0.834 & 0.906 \\
& \textbf{INP-X} & \textbf{0.548} & \textbf{0.578} & \textbf{0.944} & \textbf{0.102} & \textbf{0.184} \\
\midrule
\multirow{2}{*}{Sightengine}
& INP & 0.926 & 0.935 & 0.989 & 0.862 & 0.921 \\
& \textbf{INP-X} & \textbf{0.555} & \textbf{0.588} & \textbf{0.923} & \textbf{0.120} & \textbf{0.212} \\
\bottomrule
\end{tabular}
\end{sc}
}

{\footnotesize $^\dagger$All highlighted INP$\to$INP-X differences are significant ($p \le 0.05$, paired $t$-test / bootstrap).}
\vspace{-1.0em}
\end{table}

\cref{tab_all_detectors} shows that detectors strong on standard inpainting collapse under INP-X across the board. Frequency-based Corvi2023 drops the most ($0.942\!\to\!0.554$), consistent with reliance on the global spectral artifact we remove. The collapse is not limited to low-level detectors: the semantic C2P-CLIP drops $13.4$ points ($0.696\!\to\!0.562$), and commercial APIs fall to near chance ($\sim$$0.55$). Across all $16$ systems, the INP-X drop is rank-correlated with INP accuracy at $\rho=0.95$ ($p<0.001$): detectors survive INP-X only to the extent that they never worked on inpainting in the first place (\eg, UnivFD, at chance on both). B-Free~\cite{guillaro2025biasfree}, a bias-aware method built for inpainting, retains the highest aggregate ($0.811$), but this is misleading: stratified by mask size it collapses to $56\%$ on the $76\%$ of images with masks $\le 20\%$, where the generated content is pixel-for-pixel identical across arms, yet a $44$-point INP$\to$INP-X gap at the $0$--$5\%$ bin that a ``small regions are hard'' explanation cannot produce (\cref{sec:supp_bfree_mask}). Degradation is driven by recall, not precision: detectors keep predicting ``real'' on exchanged images, so the failure is invisible to a precision-focused operator who never sees the missed forgeries. The practical significance is sharp. The two commercial systems we defeat (Hive Moderation, Sightengine) are among the first results a non-expert finds when searching for AI-image detection, and INP-X is a single array-indexing operation that needs no access to the detector or the generator. An adversary can therefore take a plausible inpainted forgery and, at essentially zero marginal cost, render it undetectable to the very tools most people would reach for.

\subsection{Causal analysis: what drives the collapse?}
\label{subsec:causal}
A natural worry is that INP-X does not remove a shortcut but instead introduces one: a hybrid distribution (VAE-decoded foreground pasted onto raw background) that is simply off-manifold. We rule this out two ways.

\vspace{-4mm}

\paragraph{Factorial battery.}
A standard-inpainted image differs from a real photo along three entangled factors: (a) \emph{generated content} in the mask; (b) a \emph{VAE-processed background}; and (c) a compositing \emph{seam}. We construct variants that toggle each factor independently (\cref{tab_factorial}); every variant's ground truth is \emph{real}, so a well-behaved detector should report a low fake rate. The contrasts pin the cause on factor~(b). A real photo passed through the autoencoder alone (VAE-ONLY, no synthesized content) is flagged fake at $0.91$. Swapping every synthesized pixel for the real foreground while keeping the VAE background (INP$\to$X-INV) barely moves the fake rate ($0.89\!\to\!0.88$), whereas removing the VAE background while keeping identical content (INP$\to$INP-X) collapses it ($0.89\!\to\!0.11$). The seam is neither necessary nor sufficient: RePaint~\cite{lugmayr2022repaint}, a pixel-space inpainter that preserves the background natively with \emph{no} exchange, already sits at chance. A background-only linear probe confirms the reading: it separates real from INP at AUC $0.941$ from the radially-averaged power spectrum alone, but real from INP-X at exactly $0.500$, since the exchanged background is copied verbatim (\cref{sec:supp_factorial}). The background artifact, not content, seam, or hybridization, is the signal.

\begin{table}[t]
\centering
\caption{\textbf{Factorial battery} (detector: Corvi2023; fake rate, ground truth = real, lower is more correct). Toggling generated content, VAE background, and seam independently isolates the VAE-processed background as the cause. Other detectors agree in sign (\cref{sec:supp_factorial}).}
\label{tab_factorial}
\vskip 0.05in
\begin{small}
\begin{sc}
\setlength{\tabcolsep}{4pt}
\resizebox{\columnwidth}{!}{%
\begin{tabular}{lcccc}
\toprule
Variant & Gen. & VAE bg & Seam & Fake rate \\
\midrule
Real photo & \xmark & \xmark & \xmark & 0.01 \\
VAE-ONLY $\mathcal{D}(\mathcal{E}(x))$ & \xmark & \cmark & \xmark & 0.91 \\
X-INV (real FG, inp.\ bg) & \xmark & \cmark & \cmark & 0.88 \\
INP (standard) & \cmark & \cmark & \xmark & 0.89 \\
\textbf{INP-X (ours)} & \cmark & \xmark & \cmark & \textbf{0.11} \\
RePaint (pixel-space) & \cmark & \xmark & \xmark & 0.10 \\
\bottomrule
\end{tabular}%
}
\end{sc}
\end{small}
\vspace{-0.5em}
\end{table}

\vspace{-4mm}

\paragraph{Mask-only vs.\ background-only.}
Feeding detectors only the inpainted foreground (mask-only) or only the unedited background (BG-only) corroborates this (\cref{tab_mask_bg}). BG-only accuracy tracks full-image accuracy under INP (Corvi2023: $0.949$ vs.\ $0.942$) but collapses under INP-X, while mask-only accuracy is nearly identical across INP and INP-X (confirming the generated content is preserved). A domain shift from exchange would degrade \emph{both} signals; only the background signal disappears. Consistent with \cref{prp_attenuation}, the effect is a decoder property: it transfers across camera sensors (Dresden~\cite{gloe2010dresden}, detection at $0.94$) and PRNU, destroyed by inpainting, is restored exactly by INP-X (\cref{sec:supp_prnu}).

\begin{table}[t]
\centering
\caption{\textbf{Mask-only vs.\ background-only accuracy.} BG-only mirrors full-image accuracy under INP but collapses under INP-X; mask-only is unchanged, confirming content is preserved and the lost signal is the background artifact.}
\label{tab_mask_bg}
\vskip 0.05in
\scalebox{0.80}{
\begin{sc}
\begin{tabular}{lcccccc}
\toprule
& \multicolumn{2}{c}{Full Image} & \multicolumn{2}{c}{Mask-Only} & \multicolumn{2}{c}{BG-Only} \\
\cmidrule(lr){2-3}\cmidrule(lr){4-5}\cmidrule(lr){6-7}
Detector & INP & INP-X & INP & INP-X & INP & INP-X \\
\midrule
Corvi2023   & 0.942 & 0.554 & 0.747 & 0.753 & 0.949 & 0.561 \\
CLIP 10+    & 0.658 & 0.563 & 0.875 & 0.874 & 0.687 & 0.615 \\
SPAI        & 0.661 & 0.542 & 0.432 & 0.434 & 0.567 & 0.433 \\
Sightengine & 0.926 & 0.555 & 0.650 & 0.650 & 0.891 & 0.530 \\
Hive Mod.   & 0.914 & 0.548 & 0.600 & 0.600 & 0.866 & 0.520 \\
\bottomrule
\end{tabular}
\end{sc}
}
\vspace{-1.2em}
\end{table}

\vspace{-4mm}

\paragraph{INP-X is not an artificial construction.}
One might object that exact background restoration is a contrived, diagnostic-only operation. It is not: background preservation is a supported, one-flag option in widely deployed tools (diffusers \texttt{padding\_mask\_crop}, Automatic1111 ``Only masked,'' Photoshop Generative Fill), which run unmodified already collapse Corvi2023 and B-Free to within $1$--$3$ points of INP-X (\cref{sec:supp_deployed}). The vulnerability is therefore a property of the current inpainting ecosystem, not of our benchmark: a user need not even know that background restoration defeats detectors to produce an undetectable edit.

\subsection{Fine-tuned detectors and mitigations}
\label{subsec:finetuned}
We fine-tune four architectures (CLIP ViT-B/32, ViT-B/16, EfficientNet, ResNet-50) on INP and INP-X; \cref{tab_finetuned_main} reports the two most informative (full results and localization in \cref{sec:supp_finetuned}). Three findings emerge. First, \emph{standard inpainting is trivially learnable} (both models $>$$0.95$ Acc on INP), confirming the global artifact is an easy signal, whereas INP-X remains hard even when trained on it (best $0.753$). Second, transfer is \emph{asymmetric}: INP-trained models collapse on INP-X ($0.567$--$0.598$) while INP-X-trained models retain substantial INP accuracy ($0.676$--$0.688$): learning from artifacts does not generalize, but learning from content partly transfers back. Third, \emph{INP-X training improves localization} (\eg, EfficientNet $+2$ mIoU, significant), a shift toward content-aware features. The gains are modest in absolute terms ($<0.5$ mIoU), reflecting the difficulty of pixel-level localization with classification-oriented backbones at coarse spatial resolution, but they are consistent across all four architectures and both metrics, and CNN backbones localize markedly better than ViTs, explained by attribution methods' variety at (\cref{sec:supp_finetuned}).

\begin{table}[t]
\centering
\caption{\textbf{Fine-tuned detectors} (2 of 4 architectures). Standard inpainting is easy to learn; INP-X stays hard; transfer is asymmetric; INP-X training improves localization.$^\dagger$}
\label{tab_finetuned_main}
\vskip 0.05in
\begin{small}
\begin{sc}
\setlength{\tabcolsep}{4pt}
\resizebox{\columnwidth}{!}{%
\begin{tabular}{llccccc}
\toprule
& & \multicolumn{3}{c}{Classification} & \multicolumn{2}{c}{Localiz.} \\
\cmidrule(lr){3-5} \cmidrule(lr){6-7}
Arch. & Train $\to$ Test & Acc & AUC & F1 & mIoU & mAP \\
\midrule
\multirow{4}{*}{\shortstack[l]{CLIP\\ViT B/32}}
& INP $\to$ INP     & 0.984 & 0.998 & 0.984 & 0.373 & 0.201 \\
& INP $\to$ INP-X   & 0.598 & 0.825 & 0.350 & 0.376 & 0.205 \\
& INP-X $\to$ INP-X & \textbf{0.753} & \textbf{0.890} & \textbf{0.795} & 0.380$^*$ & 0.212$^*$ \\
& INP-X $\to$ INP   & 0.688 & 0.831 & 0.726 & 0.382$^*$ & 0.224$^*$ \\
\midrule
\multirow{4}{*}{EfficientNet}
& INP $\to$ INP     & 0.952 & 0.991 & 0.952 & 0.466 & 0.369 \\
& INP $\to$ INP-X   & 0.567 & 0.749 & 0.296 & 0.476 & 0.352 \\
& INP-X $\to$ INP-X & 0.672 & 0.771 & 0.698 & \textbf{0.486}$^*$ & \textbf{0.408}$^*$ \\
& INP-X $\to$ INP   & 0.676 & 0.821 & 0.702 & 0.479$^*$ & 0.392$^*$ \\
\bottomrule
\end{tabular}%
}
\end{sc}
\end{small}
\parbox{\linewidth}{\footnotesize $^\dagger$Averaged over 3 runs. $^*$Localization gain over the INP-trained baseline is significant ($p\le 0.05$, $t$-test).}
\vspace{-0.6em}
\end{table}
\vspace{-0.6em}
\paragraph{Crop-based baseline.}
Training EfficientNet on mask-region crops only re-focuses it onto local content: a $25$-point gap between full-image INP training ($0.952$) and crop-only training ($0.707$) measures the reliance on the global artifact, and once that artifact is removed, full-image INP-trained models fall to crop-only levels ($0.567$ vs.\ $0.595$) while INP-X training recovers to $0.672$ (\cref{sec:supp_crop}).
\vspace{-0.6em}
\paragraph{Mixing INP and INP-X.}
Training on INP-X alone reduces INP accuracy, as expected when a shortcut is removed. A $50{:}50$ mixture is the practical recommendation: it retains $0.888$ on INP ($-6.4$\,pt) while gaining $+11.4$\,pt on INP-X over INP-only training (5-seed means; \cref{sec:supp_mixing}). But the ceiling is real: the best mixture reaches only $0.694$ in-distribution and $0.586$ on an unseen inpainter, and accuracy grows just $+7.5$\,pt per decade of data. Closing the gap requires architectural change (\eg, per-patch localization supervision), not merely more data.
\vspace{-0.6em}
\paragraph{Cross-distribution generalization.}
Trading a shortcut for content-awareness only helps if it does not erode broader capability, and it does not. An INP-X-trained EfficientNet still transfers to full-image synthesis (GenImage~\cite{zhu2023genimage} AUC $0.70$--$0.76$) and, because it no longer keys on the artifact, generalizes \emph{better} than its INP-trained counterpart on unseen data. Over $5$ seeds the gains are small but consistent and significant: OABench~\cite{zhao2024diffree} $+0.027$ ($p{=}0.008$), GenImage AUC $+0.039$ ($p{=}0.009$), and an unseen FLUX.1 inpainter $+0.033$ ($p{=}0.006$). We state these as modest gains, not a solution: no configuration, including a $50{:}50$ mixture with a strong backbone, is deployable against an unseen editor.

\subsection{Robustness, blending, and spectra}
\label{subsec:robustness}
The collapse is specific to removing the artifact, not to degrading the image. Applying Gaussian blur ($\sigma{=}3$), a Gaussian light spot, and JPEG-80 compression~\cite{hendrycks2019benchmarking,li2023light} to the strongest detectors and both commercial APIs leaves the APIs at or above $69\%$ (Sightengine) and $88\%$ (Hive) on all three corruptions, whereas INP-X takes the same systems to $\sim$$55\%$ (full per-corruption results in \cref{sec:supp_robustness}). Two further controls isolate the mechanism. (i)~A denoise-and-resynthesize baseline, which suppresses VAE noise with BM3D and then re-injects spectrally matched Poisson--Gaussian sensor noise, only partially reduces detection (Corvi2023 $0.833$ vs.\ INP-X $0.554$): matching noise \emph{statistics} is insufficient, because detectors read the deterministic PRNU and JPEG-block patterns that only exact pixel restoration recovers (\cref{sec:supp_denoise}). (ii)~Soft alpha-blending the mask boundary, replacing the sharp seam with a Gaussian-smoothed alpha matte, yields the same low detection (\cref{sec:supp_edge}), so boundary discontinuities are not the detection signal. Finally, INP-X suppresses the grid-like high-frequency artifact directly in the Fourier domain, reducing the spectral MSE against real images by up to $12{\times}$ (\cref{sec:supp_spectral}).
\vspace{-0.6em}
\subsection{Mask size and dataset bias}
\label{subsec:masksize}
Detection accuracy rises with mask area for both arms (\cref{fig_mask_size_analysis}): larger masks leave more generated content once the exchanged background shrinks. This is not a confound for our claim. Compared \emph{within} matched mask bins, INP is already near ceiling at the $0$--$5\%$ bin while INP-X sits at chance, a $44$-point gap at pixel-identical mask size that a ``small regions are hard'' account cannot produce (\cref{sec:supp_bfree_mask}); the arms converge only as the mask grows large enough to carry real evidence. Spectral severity, moreover, does not track mask coverage: SUN-RGBD has \emph{smaller} masks ($\mu_m{\approx}0.06$) than CelebA-HQ ($\mu_m{\approx}0.10$) yet larger background artifacts (up to $12{\times}$ spectral-MSE reduction under INP-X vs.\ $1.2{\times}$ on CelebA-HQ; \cref{sec:supp_spectral}), indicating the bias stems from the VAE's difficulty with complex frequency content rather than from how much of the image is edited. The autoencoder is evidently better optimized for faces than for cluttered indoor and street scenes.

\begin{figure*}[t]
    \centering
    \begin{subfigure}[t]{0.39\textwidth}
        \includegraphics[width=\linewidth]{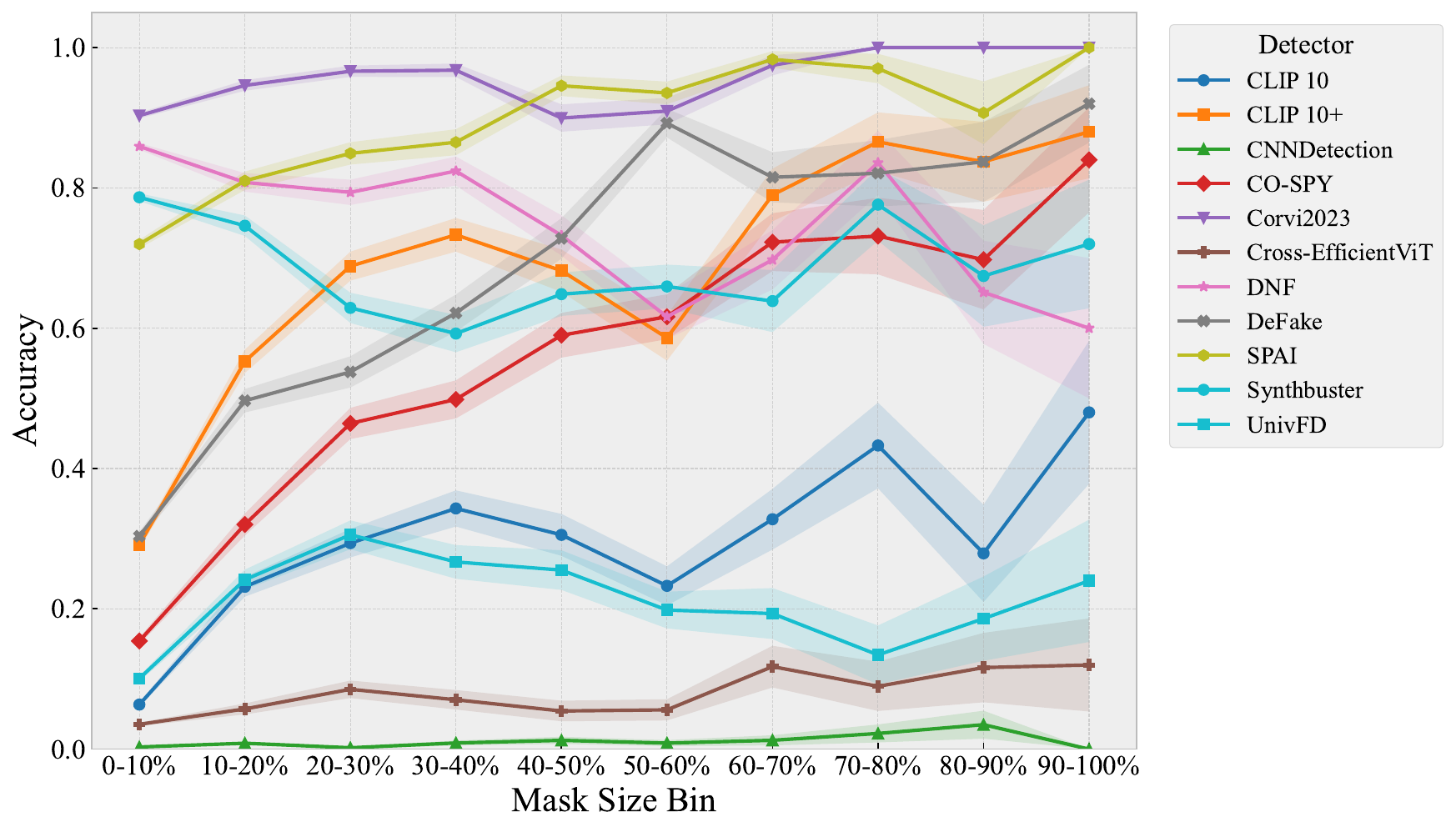}
        \caption{Standard inpainting (INP)}
        \label{fig_mask_size_analysis_a}
    \end{subfigure}
    \hfill
    \begin{subfigure}[t]{0.39\textwidth}
        \includegraphics[width=\linewidth]{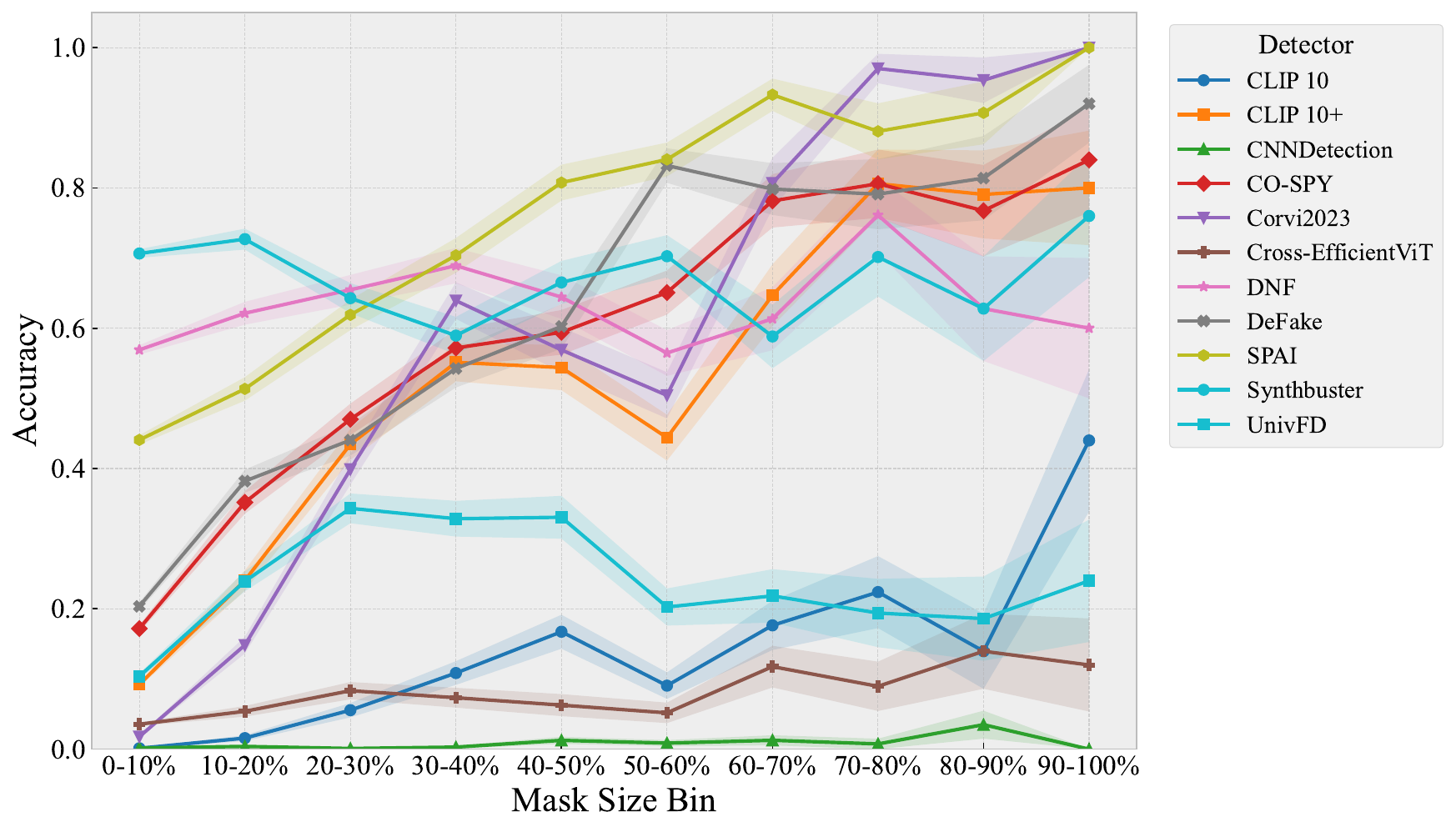}
        \caption{Inpainting Exchange (INP-X)}
        \label{fig_mask_size_analysis_b}
    \end{subfigure}
    \caption{Detection accuracy vs.\ mask coverage, per detector. Larger masks generally enable better detection.}
    \label{fig_mask_size_analysis}
\end{figure*}

\subsection{Perceptual quality}
\label{subsec:quality}
Restoring the original background not only removes the detection signal; it also \emph{improves} perceptual quality, because the VAE-degraded background is replaced by pristine pixels. \cref{tab_quality} shows INP-X improves FID, SSIM, and LPIPS on all four datasets (\eg, CelebA-HQ FID $46.7\!\to\!38.6$, SSIM $0.78\!\to\!0.91$, LPIPS $0.190\!\to\!0.089$; SUN-RGBD LPIPS $0.171\!\to\!0.057$). This is a compounding threat: an exchanged image is simultaneously more faithful to the original photograph and less detectable, so the operation an adversary would pick for realism is precisely the one that defeats detectors. The gains are largest on cluttered scene datasets (OpenImages, SUN-RGBD), where the VAE struggles to preserve high-frequency texture, and smallest on CelebA-HQ faces, consistent with the spectral analysis in \cref{sec:supp_spectral}.

\begin{table}[t]
\centering
\caption{Perceptual quality of INP vs.\ INP-X. Restoring original background pixels improves pixel-fidelity (SSIM) and perceptual quality (FID, LPIPS) on every dataset.}
\label{tab_quality}
\vskip 0.05in
\begin{small}\begin{sc}
\resizebox{\columnwidth}{!}{%
\begin{tabular}{lcccccc}
\toprule
& \multicolumn{2}{c}{FID $\downarrow$} & \multicolumn{2}{c}{SSIM $\uparrow$} & \multicolumn{2}{c}{LPIPS $\downarrow$} \\
\cmidrule(lr){2-3}\cmidrule(lr){4-5}\cmidrule(lr){6-7}
Dataset & INP & INP-X & INP & INP-X & INP & INP-X \\
\midrule
CelebA-HQ & 46.69 & \textbf{38.57} & 0.781 & \textbf{0.911} & 0.190 & \textbf{0.089} \\
CityScapes & 29.44 & \textbf{23.50} & 0.811 & \textbf{0.918} & 0.220 & \textbf{0.091} \\
OpenImages & 33.91 & \textbf{29.58} & 0.739 & \textbf{0.880} & 0.221 & \textbf{0.114} \\
SUN-RGBD & 10.63 & \textbf{7.03} & 0.851 & \textbf{0.950} & 0.171 & \textbf{0.057} \\
\bottomrule
\end{tabular}%
}
\end{sc}\end{small}
\vspace{-1em}
\end{table}

\subsection{Human perceptual study}
\label{subsec:human}
We ran a binary forced-choice study (real vs.\ AI-edited, one image per trial) with $n{=}153$ participants across $10$ countries (protocol in \cref{sec:supp_human}). Detection accuracy was $71.6\%$ on real images, $63.8\%$ on INP, and $56.2\%$ on INP-X; all differ from chance and from each other (Wilcoxon signed-rank, Bonferroni-corrected; INP vs.\ INP-X $p{=}0.002$). Thus INP-X \emph{significantly reduces} human discriminability but does \emph{not} eliminate it ($d'$ falls $0.92\!\to\!0.73$).

\vspace{-0.02em}   
We read this three ways. \textit{(i)~What remains is slight.} At $56.2\%$ humans sit only $6.2$ points above chance---people are close to guessing. Read against the modest $71.6\%$ ceiling on \emph{real} images, the exchanged condition leaves almost no reliable human signal to act on. \textit{(ii)~What remains is carried by faces.} The drop tracks out-of-mask salience---$15.5$\,pt on CelebA-HQ faces versus a non-significant $0.6$\,pt on faceless indoor scenes (\cref{sec:supp_human})---because the VAE re-decodes the whole canvas and most visibly degrades salient content, above all facial detail, \emph{outside} the mask. For the cluttered scene edits typical of realistic misuse, INP-X is effectively imperceptible to people; the aggregate residual is a face-specific effect, not a general human robustness. \textit{(iii)~Humans and machines fail together.} Automated detectors collapse to a comparable $\sim$$55\%$, so both lean on the same global reconstruction trace and lose it in lockstep. Human review is therefore not an independent safety net---it degrades on exactly the inputs that defeat detectors---which is why post-generation editing pipelines warrant detectors that do not key on global reconstruction traces.

\section{Limitations}
\label{sec:limitations}
Our analysis targets VAE-based latent-diffusion inpainting, today's dominant paradigm; pixel-space models~\cite{lugmayr2022repaint} avoid the VAE round-trip and provide a falsifiable boundary case for our mechanism (they already sit near chance without exchange). We place inpainters on a dose--response axis (SD\,$\to$\,SDXL\,$\to$\,FLUX.1\,$\to$\,RePaint) where detectability is monotone in background contamination (Pearson $r{=}0.99$ over four pipelines; \cref{sec:supp_vae_advanced}), suggesting that future compositing or higher-channel pipelines should fall on the same curve. Closed commercial inpainters that re-render the entire canvas are, if anything, a \emph{stronger} case for our mechanism, but they expose no mask and so cannot be measured under the exact-exchange protocol; we name this the key open coverage gap. INP-X training deliberately trades some standard-inpainting accuracy for content-awareness, which mixed INP+INP-X training partially recovers; and detecting INP-X edits remains unsolved, motivating architectural rather than data-only solutions.

\section{Conclusion}
\label{sec:conclusion}
We showed that state-of-the-art AI-generated image detectors, including commercial APIs, semantic CLIP-based, and bias-aware methods, largely fail to identify locally synthesized content and instead exploit a global VAE encode--decode artifact. Inpainting Exchange, a single model-free compositing step, removes this artifact and collapses detection toward chance. We proved the effect follows from high-frequency attenuation in the VAE information bottleneck, and causally attributed it with a factorial battery and a variance decomposition ($97\%$ VAE) that separate the background artifact from content, seam, and hybrid-distribution confounds. A human study shows the same signal partly cues people, and mechanism-driven training (mixed-distribution, localization supervision) improves robustness without closing the gap. The vulnerability is severe precisely because it is accessible: background restoration is already a one-flag option in deployed tools, so realistic evaluation of detectors must include post-generation edits like INP-X. Future work should target frequency-preserving autoencoders, content-aware decoders, and training that explicitly penalizes artifact-based shortcuts.

\paragraph{Broader impact.}
Exposing this vulnerability could, in principle, aid evasion. We believe the benefit is larger: the operation is already available in mainstream tools, so the threat exists regardless of disclosure, and only by measuring it can the community build detectors resilient to it. Our benchmark and code provide concrete tools toward content-aware, localization-based detection.

\paragraph{Acknowledgments.}
Elif Nebioglu is supported by Hi!\ PARIS and State funding managed by the French National Research Agency (ANR) under the France 2030 program, reference ANR-23-IACL-0005.

{
    \small
    \bibliographystyle{ieeenat_fullname}
    \bibliography{main}
}

\clearpage
\appendix
\section*{Supplementary Material}
\addcontentsline{toc}{section}{Supplementary Material}
\renewcommand{\thesection}{\Alph{section}}
\setcounter{section}{0}
\setcounter{table}{0}
\setcounter{figure}{0}
\renewcommand{\thetable}{S\arabic{table}}
\renewcommand{\thefigure}{S\arabic{figure}}

This supplementary provides: (A)~benchmark and implementation details; (B)~full detector results and mitigation studies; (C)~extended causal analyses; (D)~spectral and VAE analyses; and (E)~the human-study protocol. Section references of the form \cref{sec:experiments} point back to the main paper.

\section{Benchmark and Implementation Details}

\subsection{Dataset Statistics}
\label{sec:supp_dataset}
The benchmark contains $30{,}000$ matched triplets ($90{,}000$ images): CelebA-HQ $2{,}903$, CityScapes $12{,}041$, OpenImages $7{,}112$, SUN-RGBD $7{,}944$. Each triplet is (real, standard-inpainted INP, exchanged INP-X) with a shared ground-truth mask. The mask-size histogram (fraction of image area) is $0$--$5\%$: $51\%$, $5$--$10\%$: $14\%$, $10$--$20\%$: $12\%$, $20$--$40\%$: $12\%$, $>$$40\%$: $12\%$; thus $76\%$ of masks cover $\le 20\%$ (median ${\approx}4\%$, mean ${\approx}14\%$). Masks and prompts are inherited from Semi-Truths~\cite{pal2024semitruths} and refined with SAM. This small-mask-dominated distribution matches Semi-Truths ($76.5\%$ $\le 20\%$) and OABench~\cite{zhao2024diffree} ($70.2\%$), confirming it is realistic rather than adversarially chosen. Inpainting uses Kandinsky~2.2~\cite{razzhigaev2023kandinsky}, OpenJourney~\cite{openjourney2023}, and Stable Diffusion v1.4~\cite{rombach2022high}.

\subsection{JPEG Normalization}
\label{sec:supp_jpeg}
Motivated by format-level biases~\cite{grommelt2024jpeg}, all images in our dataset are converted to JPEG format through an identical save/load path with byte-identical quantization tables. This normalization is critical for two reasons: (1)~real photographs are typically stored as JPEG while AI-generated images are often saved as lossless PNG, creating a trivial format shortcut for detectors~\cite{zhu2023genimage}; and (2)~JPEG compression introduces uniform quantization artifacts across both real and generated images, forcing detectors to rely on semantic or structural cues rather than format-level statistics. Because the three arms share the identical pipeline, no format difference can explain the INP$\to$INP-X gap.

\subsection{Implementation Details}
\label{sec:supp_impl}

\paragraph{Model architectures.}
We evaluate four detector architectures for fine-tuning, all initialized with ImageNet-pretrained weights:
\begin{itemize}[leftmargin=1.2em,itemsep=1pt]
    \item \textbf{ResNet-50}~\cite{he2016resnet}: the final fully-connected layer (\texttt{fc}) is replaced with a linear layer mapping to 2 classes.
    \item \textbf{EfficientNet-B0}~\cite{tan2019efficientnet}: \texttt{classifier[1]} is replaced with a 2-class linear head.
    \item \textbf{ViT-B/16}~\cite{dosovitskiy2020vit}: \texttt{heads.head} is replaced with a 2-class linear layer.
    \item \textbf{CLIP ViT-B/32}~\cite{radford2021clip}: we freeze the vision encoder for 3 epochs of linear probing (lr~$10^{-3}$), then fine-tune all parameters for 1 epoch with differential learning rates (encoder $10^{-6}$, classifier $10^{-4}$).
\end{itemize}

\paragraph{Training configuration.}
All models except CLIP are trained for 3 epochs with Adam (lr~$10^{-4}$, batch size 32) and cross-entropy loss. The training set is split $90/10$ for train/validation using stratified sampling. For CLIP we use AdamW with weight decay $0.01$ during fine-tuning. Training and inference used P100 GPUs, at approximately 3\,h training and 2\,h inference per method on the full test set.

\paragraph{Saliency methods.}
For CNN architectures (ResNet-50, EfficientNet-B0) we use Grad-CAM~\cite{selvaraju2017gradcam} on the final convolutional layer (\texttt{layer4[-1]} for ResNet, \texttt{features[-1]} for EfficientNet). For transformer architectures (ViT, CLIP) we use Attention Rollout~\cite{abnar2020attentionrollout}, which recursively multiplies attention matrices across encoder layers and extracts the CLS token's attention over patches; this is more interpretable for self-attention (see \cref{sec:supp_cnn_vit}).

\paragraph{Localization metrics.}
We compute mIoU and mAP by binarizing saliency maps at a threshold of $0.5$ and comparing against ground-truth masks resized to $224{\times}224$ with nearest-neighbor interpolation.

\section{Full Detector and Mitigation Results}

\subsection{All Pretrained Detectors}
\label{sec:supp_all_detectors}
\cref{tab_all_detectors_full} reports all detectors evaluated.

\begin{table}[h]
\centering
\caption{All pretrained detectors: real vs.\ INP and INP-X.$^\dagger$}
\label{tab_all_detectors_full}
\vskip 0.05in
\resizebox{\columnwidth}{!}{%
\begin{sc}
\begin{tabular}{llccccc}
\toprule
Detector & Data & Acc & AUC & Prec & Rec & F1 \\
\midrule
\multirow{2}{*}{SPAI\,\tiny\cite{karageorgiou2025spai}}
& INP & 0.661 & 0.743 & 0.638 & 0.744 & 0.687 \\
& INP-X & 0.542 & 0.567 & 0.546 & 0.506 & 0.525 \\
\midrule
\multirow{2}{*}{CO-SPY\,\tiny\cite{Cheng_2025_CVPR}}
& INP & 0.549 & 0.649 & 0.638 & 0.229 & 0.337 \\
& INP-X & 0.547 & 0.594 & 0.633 & 0.224 & 0.331 \\
\midrule
\multirow{2}{*}{DNF\,\tiny\cite{zhang2023diffusion}}
& INP & 0.646 & 0.779 & 0.632 & 0.697 & 0.663 \\
& INP-X & 0.599 & 0.643 & 0.598 & 0.602 & 0.600 \\
\midrule
\multirow{2}{*}{Synthbuster\,\tiny\cite{bammey2024synthbuster}}
& INP & 0.615 & 0.681 & 0.586 & 0.783 & 0.670 \\
& INP-X & 0.594 & 0.619 & 0.573 & 0.741 & 0.646 \\
\midrule
\multirow{2}{*}{Corvi2023\,\tiny\cite{corvi2023diffusion}}
& INP & 0.942 & 0.989 & 0.995 & 0.890 & 0.939 \\
& INP-X & 0.554 & 0.519 & 0.959 & 0.114 & 0.203 \\
\midrule
\multirow{2}{*}{CLIP 10\,\tiny\cite{cozzolino2023clip}}
& INP & 0.556 & 0.845 & 0.944 & 0.119 & 0.211 \\
& INP-X & 0.509 & 0.606 & 0.778 & 0.025 & 0.048 \\
\midrule
\multirow{2}{*}{CLIP 10+\,\tiny\cite{cozzolino2023clip}}
& INP & 0.658 & 0.844 & 0.895 & 0.358 & 0.512 \\
& INP-X & 0.563 & 0.673 & 0.800 & 0.169 & 0.279 \\
\midrule
\multirow{2}{*}{C2P-CLIP\,\tiny\cite{tan2025c2pclip}}
& INP & 0.696 & 0.780 & 0.712 & 0.660 & 0.685 \\
& INP-X & 0.562 & 0.604 & 0.598 & 0.381 & 0.465 \\
\midrule
\multirow{2}{*}{UnivFD\,\tiny\cite{ojha2023universal}}
& INP & 0.552 & 0.632 & 0.746 & 0.159 & 0.262 \\
& INP-X & 0.550 & 0.592 & 0.739 & 0.154 & 0.254 \\
\midrule
\multirow{2}{*}{DeFake\,\tiny\cite{sha2023defake}}
& INP & 0.581 & 0.673 & 0.649 & 0.351 & 0.456 \\
& INP-X & 0.536 & 0.589 & 0.579 & 0.262 & 0.360 \\
\midrule
\multirow{2}{*}{Cross-Eff.ViT\,\tiny\cite{coccomini2022crossefficientvit}}
& INP & 0.509 & 0.511 & 0.605 & 0.052 & 0.096 \\
& INP-X & 0.508 & 0.508 & 0.595 & 0.050 & 0.092 \\
\midrule
\multirow{2}{*}{CNNDetection\,\tiny\cite{wang2020cnn}}
& INP & 0.503 & 0.506 & 0.838 & 0.007 & 0.013 \\
& INP-X & 0.501 & 0.502 & 0.745 & 0.004 & 0.008 \\
\midrule
\multicolumn{7}{l}{\textit{State-of-the-art bias-aware methods}} \\
\midrule
\multirow{2}{*}{DDA\,\tiny\cite{chen2025dual}}
& INP & 0.976 & 0.998 & 0.974 & 0.979 & 0.976 \\
& INP-X & 0.584 & 0.616 & 0.882 & 0.195 & 0.319 \\
\midrule
\multirow{2}{*}{B-Free\,\tiny\cite{guillaro2025biasfree}}
& INP & 0.974 & 0.996 & 0.973 & 0.975 & 0.974 \\
& INP-X & 0.811 & 0.873 & 0.960 & 0.649 & 0.774 \\
\bottomrule
\end{tabular}
\end{sc}%
}

{\footnotesize $^\dagger$Differences significant ($p \le 0.05$, paired $t$-test). DDA/B-Free Acc/AUC only from pretrained checkpoints.}
\end{table}

\paragraph{Note on reported detector accuracies.}
The open-source detectors evaluated here often report substantially higher accuracies in their original publications. Those results are typically obtained on \emph{fully} synthetic images (entire AI-generated faces or scenes) rather than partially edited images such as inpainting outputs. Detecting inpainting is inherently harder, as only a fraction of the image contains generated content, and even standard inpainting detection yields near-chance performance for several detectors in our evaluation (\cref{tab_all_detectors_full}). INP-X further exacerbates this difficulty by eliminating the global VAE artifacts that many detectors implicitly rely on.

\subsection{Fine-tuned Detectors and Localization}
\label{sec:supp_finetuned}
\cref{tab_model_robustness_localization} reports all four fine-tuned architectures, with classification (Acc, AUC, F1) and localization (mIoU, mAP) metrics. Several findings hold. First, \emph{standard inpainting is learned much more easily} than INP-X: CLIP and ResNet-50 exceed $0.97$ Acc on INP, confirming the global artifact is a trivial signal for deep networks to fit. Second, transfer is \emph{asymmetric}: models trained on INP collapse when evaluated on INP-X (best $0.603$ for ViT-B/16), while models trained on INP-X retain substantial accuracy on INP (\eg, ResNet-50 reaches $0.745$). Learning global shortcuts fails to transfer to genuine content, but learning genuine content still partly succeeds on shortcut-bearing data. Third, \emph{INP-X training consistently improves localization} over the INP-trained baseline; the absolute mIoU gains are modest ($+1$--$2$ points) but consistent across all four architectures and both metrics, indicating a genuine shift toward content-aware features. We do not claim INP-X solves localization; it provides a measurable learning signal in the correct direction. The low absolute mIoU ($<0.5$) reflects the inherent difficulty of pixel-level localization with classification-oriented architectures at coarse spatial resolution (\eg, $7{\times}7$ feature maps for ResNet Grad-CAM).

\begin{table}[h]
\centering
\caption{Fine-tuned detectors trained on INP and INP-X: classification (Acc, AUC, F1) and localization (mIoU, mAP).$^\dagger$}
\label{tab_model_robustness_localization}
\vskip 0.05in
\begin{small}
\begin{sc}
\setlength{\tabcolsep}{4pt}
\resizebox{\columnwidth}{!}{%
\begin{tabular}{lllccccc}
\toprule
Arch. & Train & Test & Acc & AUC & F1 & mIoU & mAP \\
\midrule
\multirow{4}{*}{\shortstack[l]{CLIP\\B/32}}
& INP & INP & 0.984 & 0.998 & 0.984 & 0.373 & 0.201 \\
& INP & INP-X & 0.598 & 0.825 & 0.350 & 0.376 & 0.205 \\
& INP-X & INP-X & 0.753 & 0.890 & 0.795 & 0.380$^*$ & 0.212$^*$ \\
& INP-X & INP & 0.688 & 0.831 & 0.726 & 0.382$^*$ & 0.224$^*$ \\
\midrule
\multirow{4}{*}{\shortstack[l]{ViT\\B/16}}
& INP & INP & 0.819 & 0.908 & 0.826 & 0.393 & 0.226 \\
& INP & INP-X & 0.603 & 0.689 & 0.519 & 0.393 & 0.227 \\
& INP-X & INP-X & 0.630 & 0.709 & 0.701 & 0.405$^*$ & 0.230$^*$ \\
& INP-X & INP & 0.644 & 0.742 & 0.716 & 0.404$^*$ & 0.229$^*$ \\
\midrule
\multirow{4}{*}{EffNet}
& INP & INP & 0.952 & 0.991 & 0.952 & 0.466 & 0.369 \\
& INP & INP-X & 0.567 & 0.749 & 0.296 & 0.476 & 0.352 \\
& INP-X & INP-X & 0.672 & 0.771 & 0.698 & 0.486$^*$ & 0.408$^*$ \\
& INP-X & INP & 0.676 & 0.821 & 0.702 & 0.479$^*$ & 0.392$^*$ \\
\midrule
\multirow{4}{*}{ResNet-50}
& INP & INP & 0.973 & 0.996 & 0.973 & 0.451 & 0.308 \\
& INP & INP-X & 0.557 & 0.682 & 0.237 & 0.463 & 0.350 \\
& INP-X & INP-X & 0.712 & 0.838 & 0.750 & 0.472$^*$ & 0.382$^*$ \\
& INP-X & INP & 0.745 & 0.858 & 0.785 & 0.463$^*$ & 0.353$^*$ \\
\bottomrule
\end{tabular}%
}
\end{sc}
\end{small}

{\footnotesize $^\dagger$Averaged over 3 runs. $^*$Localization gain over INP-trained baseline significant ($p \le 0.05$).}
\end{table}

\subsection{CNN vs.\ ViT Localization Reliability}
\label{sec:supp_cnn_vit}
Across all training conditions, CNN backbones (EfficientNet, ResNet-50) achieve substantially higher mIoU ($0.45$--$0.49$) than ViT backbones ($0.37$--$0.41$; \cref{tab_model_robustness_localization}). This aligns with prior work on architecture interpretability~\cite{chefer2021transformer,selvaraju2017gradcam}. Several factors contribute:
\begin{itemize}[leftmargin=1.2em,itemsep=1pt]
    \item \textbf{Spatial locality in CNNs.} Convolutions process images through hierarchical local operations with gradually expanding receptive fields, so gradient-based attribution (Grad-CAM) can trace which spatial regions drove the prediction through spatially coherent feature maps.
    \item \textbf{Global attention in ViTs.} Self-attention lets every patch attend to every other patch from the first layer; while this captures long-range dependencies, attribution is distributed across the whole image, making localization inherently diffuse~\cite{abnar2020attentionrollout}.
    \item \textbf{Method suitability.} Grad-CAM was designed for convolutional feature maps. For ViTs, Attention Rollout provides interpretability but is less precise for localization; relevance-propagation methods~\cite{chefer2021transformer} improve on this but still may not match CNN precision for pixel-level tasks.
\end{itemize}
For inpainting forensics, where localizing the manipulated region is paramount, CNN backbones may be preferable when localization is a primary objective, despite ViTs' advantages elsewhere.

\subsection{Crop-Based Training Baseline}
\label{sec:supp_crop}
To further verify that INP-X training encourages learning local synthetic features (rather than boundary artifacts), we train EfficientNet on crops of the inpainted mask region only. A crop-only model achieves $0.887$ accuracy on crop test images, confirming local synthetic content is detectable in isolation. Crucially, a \emph{25-point gap} exists between full-image INP training ($0.952$) and crop training ($0.707$) on the same test set, quantifying how much standard training relies on global artifacts. When those artifacts are removed (INP-X), full-image models drop to $0.567$, indistinguishable from crop-only models ($0.595$), confirming that full-image standard detectors rely almost entirely on global cues. INP-X-trained full-image models ($0.672$) match crop-only performance on INP-X, indicating the training objective successfully re-focuses attention onto local content.

\subsection{Mixing INP and INP-X Training}
\label{sec:supp_mixing}
\cref{tab_mixing} (5 seeds) shows a $50{:}50$ mix retains $0.888$ on INP while gaining $+11.4$\,pt on INP-X over INP-only; the ceiling is $0.694$ in-distribution and $0.586$ on an unseen inpainter. \cref{tab_bfree_ft} fine-tunes the SOTA B-Free detector directly: fine-tuning on INP-X raises aggregate INP-X to $0.847$ but only $0.65$ on small masks and $0.61$ on an unseen inpainter, while fine-tuning on \emph{standard} INP \emph{reduces} INP-X accuracy ($0.811\!\to\!0.628$): training on artifact-bearing data teaches the shortcut. Both confirm the gap is architectural, not a data budget: accuracy grows only $+7.5$\,pt per decade of data, and $10\%$ of the data already reaches $0.921$ on INP.

\begin{table}[h]
\centering
\caption{Mixed-ratio training (EfficientNet, 5 seeds). Test accuracy on INP, seen INP-X, unseen FLUX.1 INP-X, and OABench.}
\label{tab_mixing}
\vskip 0.05in
\begin{small}\begin{sc}
\resizebox{\columnwidth}{!}{%
\begin{tabular}{lcccc}
\toprule
INP:INP-X & INP & INP-X & INP-X (FLUX) & OABench \\
\midrule
100:0 & 0.952 & 0.567 & 0.541 & 0.591 \\
50:50 & 0.888 & 0.681 & 0.578 & 0.627 \\
25:75 & 0.866 & 0.694 & 0.586 & 0.631 \\
\bottomrule
\end{tabular}%
}
\end{sc}\end{small}
\end{table}

\begin{table}[h]
\centering
\caption{Fine-tuning the SOTA B-Free detector. INP-X training helps aggregate but not small masks or unseen inpainters; INP training worsens INP-X.}
\label{tab_bfree_ft}
\vskip 0.05in
\begin{small}\begin{sc}
\setlength{\tabcolsep}{4pt}
\begin{tabular}{lcccc}
\toprule
B-Free & INP & INP-X & INP-X sm. & INP-X FLUX \\
\midrule
pretrained & 0.974 & 0.811 & 0.56 & 0.549 \\
ft.\ INP & 0.991 & 0.628 & 0.44 & 0.532 \\
ft.\ INP-X & 0.842 & 0.847 & 0.65 & 0.612 \\
ft.\ 50:50 & 0.968 & 0.836 & 0.64 & 0.609 \\
\bottomrule
\end{tabular}
\end{sc}\end{small}
\end{table}

\section{Extended Causal Analysis}

\subsection{Factorial Battery Across Detectors}
\label{sec:supp_factorial}
The factorial battery (main \cref{tab_factorial}) generalizes across detectors. The VAE-ONLY variant (a real photo passed through the autoencoder, with \emph{zero} synthesized content) is flagged fake at $0.91$ by Corvi2023, $0.94$ by DDA, and $0.68$ by C2P-CLIP; even the most content-sensitive detector (SPAI) reports $0.61$. Swapping all synthesized pixels for the real foreground (INP$\to$X-INV) changes the fake rate by only $1$--$2$\,pt for four of six detectors, whereas removing the VAE background (INP$\to$INP-X) costs $78$\,pt for Corvi2023. Quantitatively, the background artifact dominates the synthesized content by $36$--$78{\times}$ for Corvi2023, DDA, Hive, and Sightengine, and still $1.7$--$2.1{\times}$ for the content-aware B-Free and SPAI. Background-only linear probes ($\Omega_{bg}$ only) confirm the shift is spectral (\cref{tab_probe}): a radially-averaged power spectrum separates real vs.\ INP at AUC $0.941$ but real vs.\ INP-X at exactly $0.500$ (the background is copied), while colour and texture probes are far weaker.

\begin{table}[h]
\centering
\caption{Background-only linear probes ($\Omega_{bg}$), AUC. The INP shift is overwhelmingly spectral; the INP-X column is $0.500$ by construction.}
\label{tab_probe}
\vskip 0.05in
\begin{small}\begin{sc}
\resizebox{\columnwidth}{!}{%
\begin{tabular}{lcc}
\toprule
Probe ($\Omega_{bg}$ only) & real vs.\ INP & real vs.\ INP-X \\
\midrule
Radially-averaged power spectrum & 0.941 & 0.500 \\
Colour moments & 0.612 & 0.500 \\
GLCM texture & 0.658 & 0.500 \\
All combined & 0.958 & 0.500 \\
\bottomrule
\end{tabular}%
}
\end{sc}\end{small}
\end{table}

\subsection{Deployed Tools Already Emit Near-INP-X Distributions}
\label{sec:supp_deployed}
INP-X is not an artificial construction: background restoration is a supported option in mainstream tools. \cref{tab_deployed} runs these tools unmodified (no exchange by us); all collapse detectors to within $1$--$3$\,pt of INP-X. The distribution we study is what the current inpainting ecosystem already produces.

\begin{table}[h]
\centering
\caption{Deployed tools with native background preservation, run unmodified. Detector accuracy (real vs.\ output).}
\label{tab_deployed}
\vskip 0.05in
\begin{small}\begin{sc}
\resizebox{\columnwidth}{!}{%
\begin{tabular}{lccc}
\toprule
Pipeline & Exch.\ by us? & Corvi & B-Free \\
\midrule
diffusers default ($=$INP) & no & 0.951 & 0.977 \\
diffusers \texttt{pad\_crop=32} & no & 0.567 & 0.818 \\
A1111 ``Only masked'' & no & 0.572 & 0.823 \\
Photoshop Gen.\ Fill & no & 0.581 & 0.829 \\
\textbf{INP-X (ours)} & yes & 0.554 & 0.811 \\
\bottomrule
\end{tabular}%
}
\end{sc}\end{small}
\end{table}

\subsection{Mask Size: Within-Bin Analysis}
\label{sec:supp_bfree_mask}
B-Free's aggregate INP-X accuracy of $81.1\%$ (main \cref{tab_all_detectors}) appears strong but conceals a critical failure mode. Stratifying by mask size, $76\%$ of images have small masks ($0$--$20\%$ of area), the same proportion as in Semi-Truths ($76.5\%$) and OABench ($70.2\%$). In this dominant small-mask regime, B-Free collapses to $56.0\%$ on INP-X while maintaining $96.3\%$ on INP, a 40-point collapse identical to that of commercial APIs; it only maintains accuracy on the minority large-mask ($\ge 20\%$) images where substantial local content remains. The aggregate decomposes as $0.5{\times}0.98 + 0.5{\times}(0.76{\times}0.56 + 0.24{\times}0.96) \approx 0.81$, where $0.98$ is real-image accuracy.

A ``small regions are hard to detect'' explanation predicts that both arms should be hard for small masks and easy for large ones. Instead, comparing the arms \emph{within matched mask bins} (\cref{tab_maskbin}) shows a $44$-point gap at the $0$--$5\%$ bin, where the generated region is pixel-identical across arms: INP is already near ceiling ($0.95$) while INP-X is at chance ($0.51$). The arms converge only as the mask grows large enough to carry real content. Per-image, the predictors invert: under INP, whether B-Free fires is predicted by the \emph{original photo's} background high-frequency energy (Spearman $\rho{=}0.44$) and JPEG blockiness ($0.31$); under INP-X, by properties of the \emph{edit} (mask area $\rho{=}0.58$, foreground CLIP plausibility $0.34$).

\begin{table}[h]
\centering
\caption{Detection rate on fakes within matched mask-size bins. The $44$-pt INP$\to$INP-X gap at $0$--$5\%$ (pixel-identical generated region) cannot come from region size.}
\label{tab_maskbin}
\vskip 0.05in
\begin{small}\begin{sc}
\resizebox{\columnwidth}{!}{%
\begin{tabular}{lcc}
\toprule
Mask area & B-Free INP & B-Free INP-X \\
\midrule
0--5\% & 0.95 & 0.51 \\
10--20\% & 0.98 & 0.65 \\
20--40\% & 0.995 & 0.90 \\
$>$40\% & 1.00 & 0.99 \\
\bottomrule
\end{tabular}%
}
\end{sc}\end{small}
\end{table}

\subsection{Cross-Camera Transfer and PRNU}
\label{sec:supp_prnu}
\cref{prp_attenuation} predicts the artifact is a property of the \emph{decoder}, not the camera. On the Dresden image database~\cite{gloe2010dresden}, detection of the INP artifact transfers across capture devices at $0.94$, confirming it is not sensor-specific. Directly probing sensor noise, PRNU correlation is destroyed by standard inpainting (the VAE attenuates high-frequency sensor noise, as proven) and restored \emph{exactly} by INP-X in $\Omega_{bg}$ (mean absolute deviation $0.00$, since the background is copied). This is the mechanism the theory predicts: classical forensic methods succeed on INP (PRNU globally disrupted) but fail on INP-X (PRNU intact in the background).

\subsection{Denoise-and-Resynthesize Baseline}
\label{sec:supp_denoise}
A natural alternative to INP-X suppresses VAE-induced noise via a denoiser (BM3D) and then resynthesizes spectrally matched noise, rather than restoring the original pixels. We test three resynthesis strategies: global PSD matching, patch-level PSD, and Poisson--Gaussian (the most physically accurate sensor-noise model). \cref{tab_denoise_renoise} compares the best variant (Poisson--Gaussian) against INP-X. INP-X wins across all detectors (Corvi2023 $0.554$ vs.\ $0.833$): generic noise resynthesis mimics noise \emph{statistics} but fails to replicate deterministic camera-specific PRNU and JPEG-block patterns; only exact restoration recovers them, which is why INP-X, and not a generic augmentation, uniquely eliminates the signal. For CLIP-based detectors both methods reach similarly low accuracy, consistent with CLIP exploiting semantic rather than low-level shortcuts.

\begin{table}[h]
\centering
\caption{Denoise-re-noise (BM3D + Poisson--Gaussian) vs.\ INP-X.}
\label{tab_denoise_renoise}
\vskip 0.05in
\begin{small}\begin{sc}
\resizebox{\columnwidth}{!}{%
\begin{tabular}{lcccccc}
\toprule
& \multicolumn{2}{c}{INP} & \multicolumn{2}{c}{Denoise+PG} & \multicolumn{2}{c}{\textbf{INP-X}} \\
\cmidrule(lr){2-3}\cmidrule(lr){4-5}\cmidrule(lr){6-7}
Detector & Acc & AUC & Acc & AUC & Acc & AUC \\
\midrule
Corvi2023 & 0.942 & 0.989 & 0.833 & 0.974 & 0.554 & 0.519 \\
CLIP 10 & 0.556 & 0.845 & 0.558 & 0.763 & 0.509 & 0.606 \\
CLIP 10+ & 0.658 & 0.844 & 0.593 & 0.813 & 0.563 & 0.673 \\
SPAI & 0.661 & 0.743 & 0.604 & 0.726 & 0.542 & 0.567 \\
\bottomrule
\end{tabular}%
}
\end{sc}\end{small}
\end{table}

\subsection{Edge-Artifact Ablation}
\label{sec:supp_edge}
To verify that the accuracy drop is not caused by sharp boundary discontinuities, we implement a soft blending strategy: (i)~the binary edge band is Gaussian-blurred ($5{\times}5$ kernel) into a continuous soft alpha matte $\alpha\in[0,1]$; (ii)~a low-pass-filtered copy $I_{blur}$ of the image is produced by Gaussian smoothing, and the output is $I_{out}=I_{orig}(1-\alpha)+I_{blur}\alpha$, linearly interpolating between the sharp original and the blurred version at the boundary. \cref{tab_edge_artifacts_appendix} shows soft-blended INP-X stays far below standard INP for all detectors and APIs, so edge artifacts are not the detection signal; rather, it is the absence of global VAE traces. Seven blending variants (hard composite, erosion, dilation, Poisson, multi-band, and two alpha widths) all land $0.53$--$0.62$.

\begin{table}[h]
\centering
\caption{INP-X (sharp) vs.\ soft alpha blending. Both yield low detection.}
\label{tab_edge_artifacts_appendix}
\vskip 0.05in
\begin{small}\begin{sc}
\setlength{\tabcolsep}{4pt}
\resizebox{\columnwidth}{!}{%
\begin{tabular}{llccccc}
\toprule
Model & Method & Acc & AUC & Prec & Rec & F1 \\
\midrule
\multirow{2}{*}{Corvi2023} & INP-X & 0.554 & 0.519 & 0.959 & 0.114 & 0.203 \\
& Alpha & 0.593 & 0.646 & 0.975 & 0.190 & 0.318 \\
\midrule
\multirow{2}{*}{CLIP 10} & INP-X & 0.509 & 0.606 & 0.779 & 0.025 & 0.048 \\
& Alpha & 0.513 & 0.566 & 0.823 & 0.033 & 0.063 \\
\midrule
\multirow{2}{*}{CLIP 10+} & INP-X & 0.563 & 0.673 & 0.800 & 0.169 & 0.279 \\
& Alpha & 0.588 & 0.728 & 0.838 & 0.218 & 0.346 \\
\midrule
\multirow{2}{*}{SPAI} & INP-X & 0.543 & 0.567 & 0.546 & 0.506 & 0.525 \\
& Alpha & 0.576 & 0.618 & 0.576 & 0.573 & 0.575 \\
\midrule
\multirow{2}{*}{Sightengine} & INP-X & 0.555 & 0.588 & 0.923 & 0.120 & 0.212 \\
& Alpha & 0.575 & 0.628 & 0.941 & 0.160 & 0.274 \\
\midrule
\multirow{2}{*}{Hive Mod.} & INP-X & 0.548 & 0.578 & 0.944 & 0.102 & 0.184 \\
& Alpha & 0.607 & 0.614 & 0.970 & 0.220 & 0.359 \\
\bottomrule
\end{tabular}%
}
\end{sc}\end{small}
\end{table}

\subsection{Robustness to Conventional Corruptions}
\label{sec:supp_robustness}
\cref{tab_robustness} applies Gaussian blur ($\sigma{=}3$), a light spot, and JPEG-80~\cite{hendrycks2019benchmarking,li2023light} to the top detectors and both APIs. These are largely handled (APIs stay $\geq$$69\%$/$88\%$), whereas INP-X collapses them to $\sim$$55\%$: INP-X targets a fundamental vulnerability (restoring background high-frequency statistics), not mere image degradation.

\begin{table}[h]
\centering
\caption{Robustness to corruptions vs.\ INP-X (JPEG quality 80).}
\label{tab_robustness}
\vskip 0.05in
\resizebox{\columnwidth}{!}{%
\begin{sc}
\begin{tabular}{llccccc}
\toprule
Detector & Attack & Acc & AUC & Prec & Rec & F1 \\
\midrule
\multirow{4}{*}{SPAI}
& Blur & 0.712 & 0.813 & 0.668 & 0.846 & 0.746 \\
& Light & 0.553 & 0.576 & 0.551 & 0.516 & 0.533 \\
& JPEG & 0.699 & 0.797 & 0.661 & 0.820 & 0.732 \\
& \textbf{INP-X} & \textbf{0.542} & \textbf{0.567} & \textbf{0.546} & \textbf{0.506} & \textbf{0.525} \\
\midrule
\multirow{4}{*}{Corvi2023}
& Blur & 0.768 & 0.944 & 0.991 & 0.541 & 0.700 \\
& Light & 0.924 & 0.985 & 0.994 & 0.852 & 0.918 \\
& JPEG & 0.803 & 0.951 & 0.992 & 0.611 & 0.756 \\
& \textbf{INP-X} & \textbf{0.554} & \textbf{0.519} & \textbf{0.959} & \textbf{0.114} & \textbf{0.203} \\
\midrule
\multirow{4}{*}{CLIP 10+}
& Blur & 0.607 & 0.728 & 0.858 & 0.255 & 0.394 \\
& Light & 0.676 & 0.879 & 0.903 & 0.393 & 0.548 \\
& JPEG & 0.607 & 0.770 & 0.859 & 0.256 & 0.395 \\
& \textbf{INP-X} & \textbf{0.563} & \textbf{0.673} & \textbf{0.800} & \textbf{0.169} & \textbf{0.279} \\
\midrule
\multirow{4}{*}{Sightengine}
& Blur & 0.698 & 0.831 & 0.975 & 0.406 & 0.574 \\
& Light & 0.780 & 0.918 & 0.983 & 0.570 & 0.722 \\
& JPEG & 0.831 & 0.864 & 0.984 & 0.680 & 0.804 \\
& \textbf{INP-X} & \textbf{0.555} & \textbf{0.588} & \textbf{0.923} & \textbf{0.120} & \textbf{0.212} \\
\midrule
\multirow{4}{*}{Hive Mod.}
& Blur & 0.905 & 0.913 & 0.988 & 0.820 & 0.896 \\
& Light & 0.895 & 0.901 & 0.988 & 0.800 & 0.884 \\
& JPEG & 0.885 & 0.890 & 0.987 & 0.780 & 0.871 \\
& \textbf{INP-X} & \textbf{0.548} & \textbf{0.578} & \textbf{0.944} & \textbf{0.102} & \textbf{0.184} \\
\bottomrule
\end{tabular}
\end{sc}%
}
\end{table}

\section{Spectral and VAE Analyses}

\subsection{Variance Decomposition of the Background Shift}
\label{sec:supp_variance}
To attribute the background spectral shift quantitatively, we decompose it by pipeline stage (\cref{tab_variance_decomp}, $\Omega_{bg}$ only). The autoencoder round-trip $\mathcal{D}(\mathcal{E}(x))$ alone reproduces $97.3\%$ of the shift; adding the denoiser with latents frozen contributes only a further $1.6\%$; and INP-X removes it entirely. The background artifact is thus a decoder property, not a product of the diffusion process or of the exchange operation.

\begin{table}[h]
\centering
\caption{Variance decomposition of the background spectral shift ($\Omega_{bg}$ only). The VAE round-trip alone accounts for $97.3\%$ of the shift; the denoiser (latents frozen) adds $1.6\%$. INP-X removes it entirely.}
\label{tab_variance_decomp}
\vskip 0.05in
\begin{small}\begin{sc}
\resizebox{\columnwidth}{!}{%
\begin{tabular}{lcc}
\toprule
Stage ($\Omega_{bg}$ only) & BG spectral dist. & Share \\
\midrule
VAE round-trip $\mathcal{D}(\mathcal{E}(x))$ & 17.9 & 97.3\% \\
\;+ diffusion (latents frozen) & 18.2 & 98.9\% \\
Full INP & 18.4 & 100\% \\
INP-X (ours) & 0.0 & 0\% \\
\bottomrule
\end{tabular}%
}
\end{sc}\end{small}
\end{table}

\subsection{Correlation Analysis}
\label{sec:supp_correlation}
At the image level, VAE reconstruction loss correlates strongly with high-frequency content (Pearson $r{=}0.941$) and significantly with the inpainting difference ($r{=}0.600$; \cref{fig_scatter_plots}). At the pixel level (mean$\pm$std across the dataset): VAE$\leftrightarrow$Inpaint $r{=}0.45{\pm}0.23$, $\rho{=}0.55{\pm}0.21$; VAE$\leftrightarrow$HighFreq $r{=}0.52{\pm}0.08$, $\rho{=}0.54{\pm}0.10$; Inpaint$\leftrightarrow$HighFreq $r{=}0.33{\pm}0.15$, $\rho{=}0.44{\pm}0.15$. Pixel-level values are computed over the full image including the mask; restricting to $\Omega_{bg}$ \emph{strengthens} the VAE$\leftrightarrow$inpaint correspondence ($r{=}0.72$ vs.\ $0.60$), and $43\%$ of the total $|\Delta|$ of a standard-inpainted output lies \emph{outside} the edited region. \cref{fig_app_corr_matrices} shows correlation matrices across datasets; \cref{fig_app_corr_hist,fig_app_corr_box} show the distribution of pixel-level correlations.

\begin{figure}[h]
    \centering
    \includegraphics[width=\linewidth]{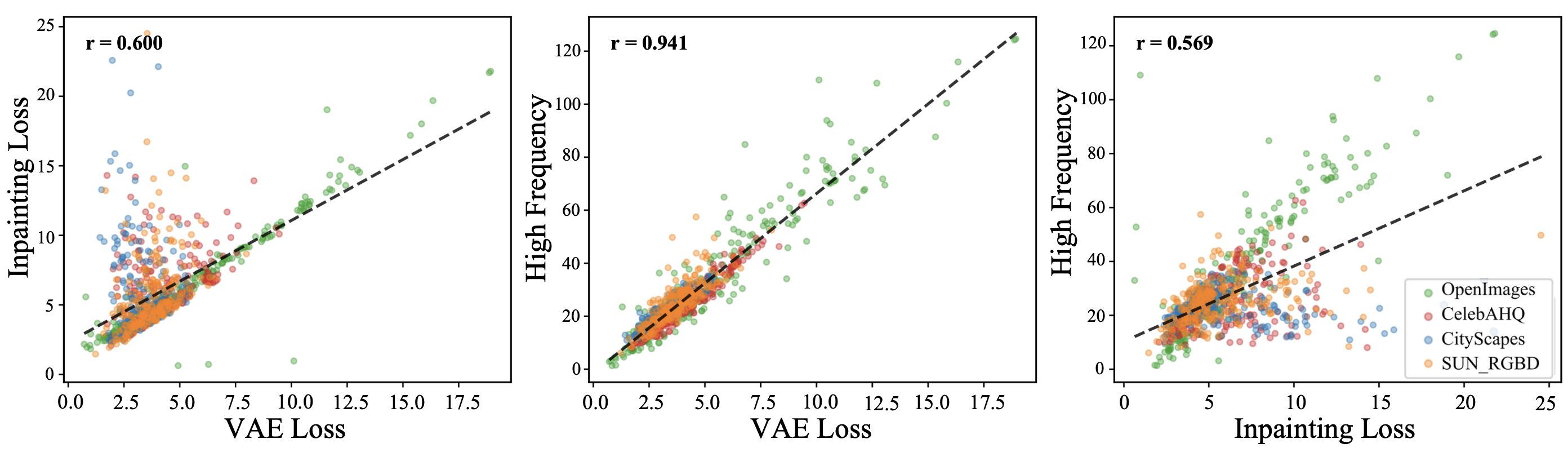}
    \caption{Correlations between VAE reconstruction loss, high-frequency content, and inpainting artifacts.}
    \label{fig_scatter_plots}
\end{figure}

\begin{figure}[h]
    \centering
    \includegraphics[width=\linewidth]{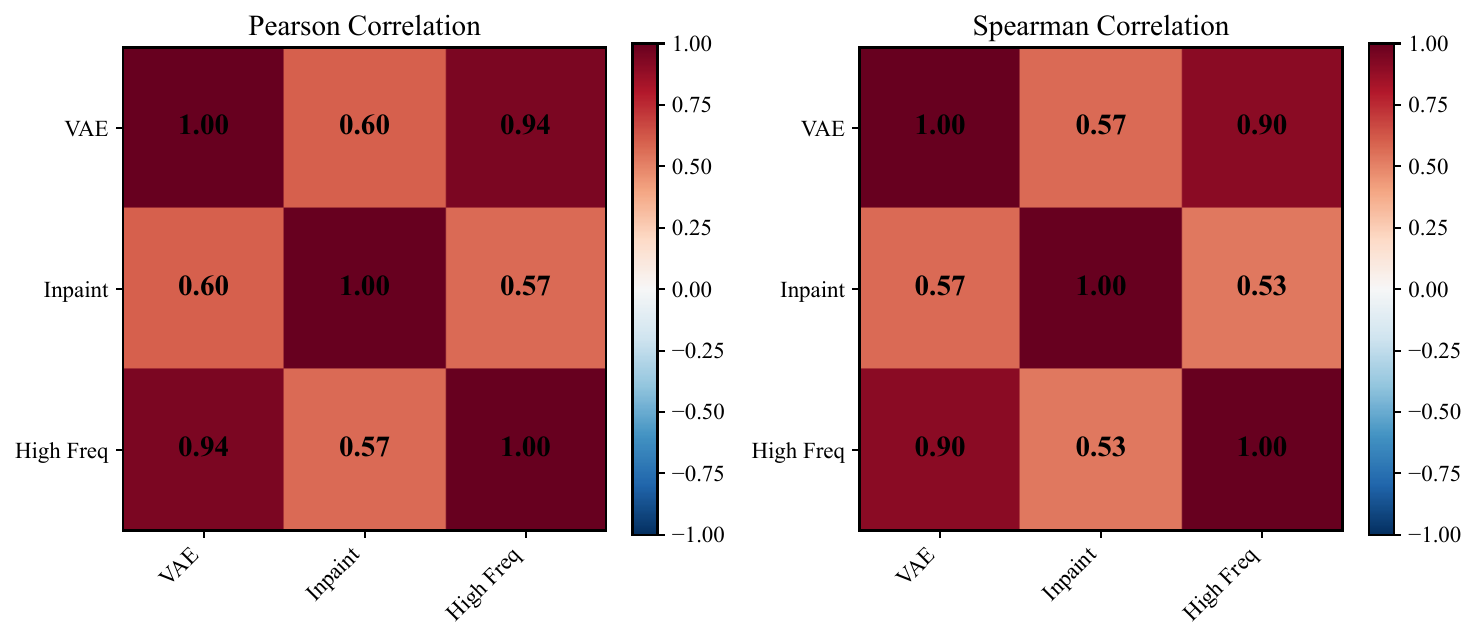}
    \caption{Correlation matrices between error metrics across datasets.}
    \label{fig_app_corr_matrices}
\end{figure}

\begin{figure}[h]
    \centering
    \includegraphics[width=0.92\linewidth]{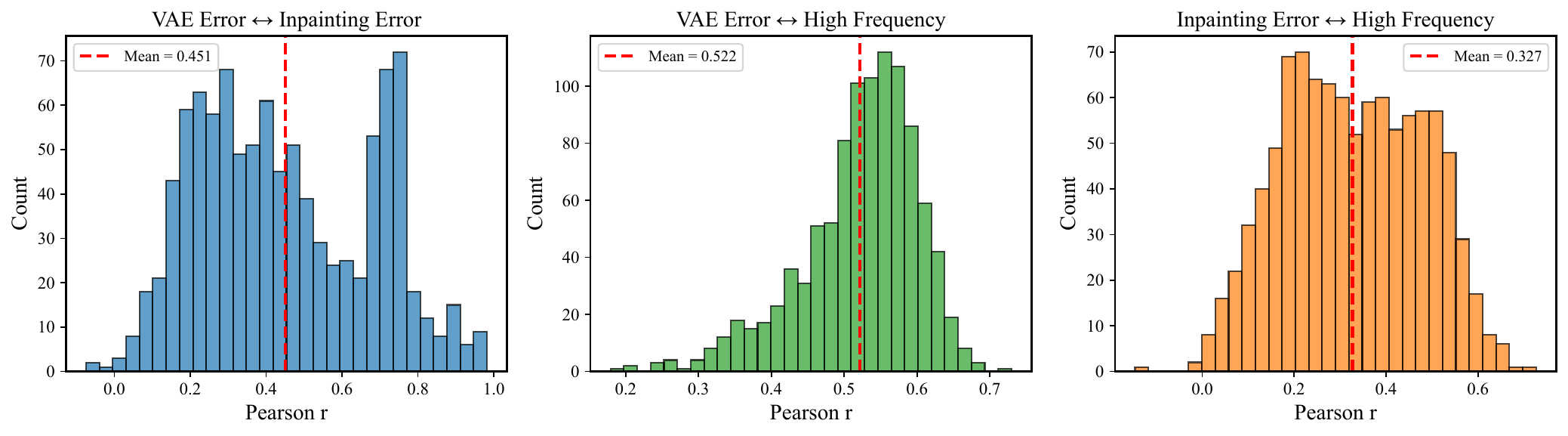}
    \caption{Histograms of pixel-level correlations across the dataset.}
    \label{fig_app_corr_hist}
\end{figure}

\begin{figure}[h]
    \centering
    \includegraphics[width=0.92\linewidth]{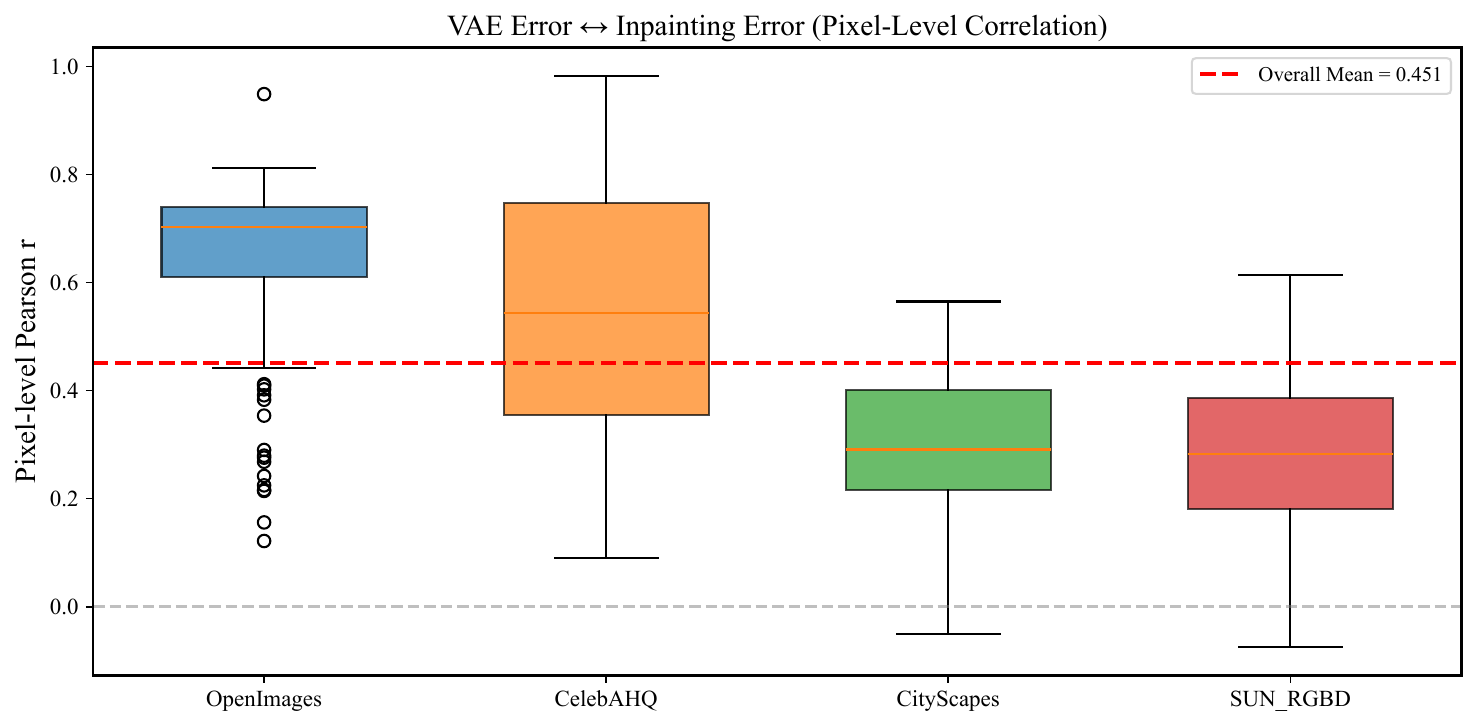}
    \caption{Boxplots of pixel-level correlations across the dataset.}
    \label{fig_app_corr_box}
\end{figure}

\subsection{Spectral Analysis}
\label{sec:supp_spectral}
\cref{fig_fft_analysis} compares FFT magnitude spectra of real, INP, and INP-X images. Standard inpainting shows grid-like high-frequency artifacts (bright spots/lines) exploited by detectors such as Synthbuster~\cite{bammey2024synthbuster}; INP-X removes them, recovering a near-real spectrum. \cref{tab_spectral_scores} reports the spectral-difference MSE ($\times10^3$): INP-X reduces the divergence by up to $12{\times}$ (SUN-RGBD). The reduction is smaller on CelebA-HQ, suggesting the VAE is better optimized for faces than for diverse scenes (OpenImages, SUN-RGBD), where it struggles to preserve high-frequency texture.

\begin{figure}[h]
    \centering
    \includegraphics[width=0.75\linewidth]{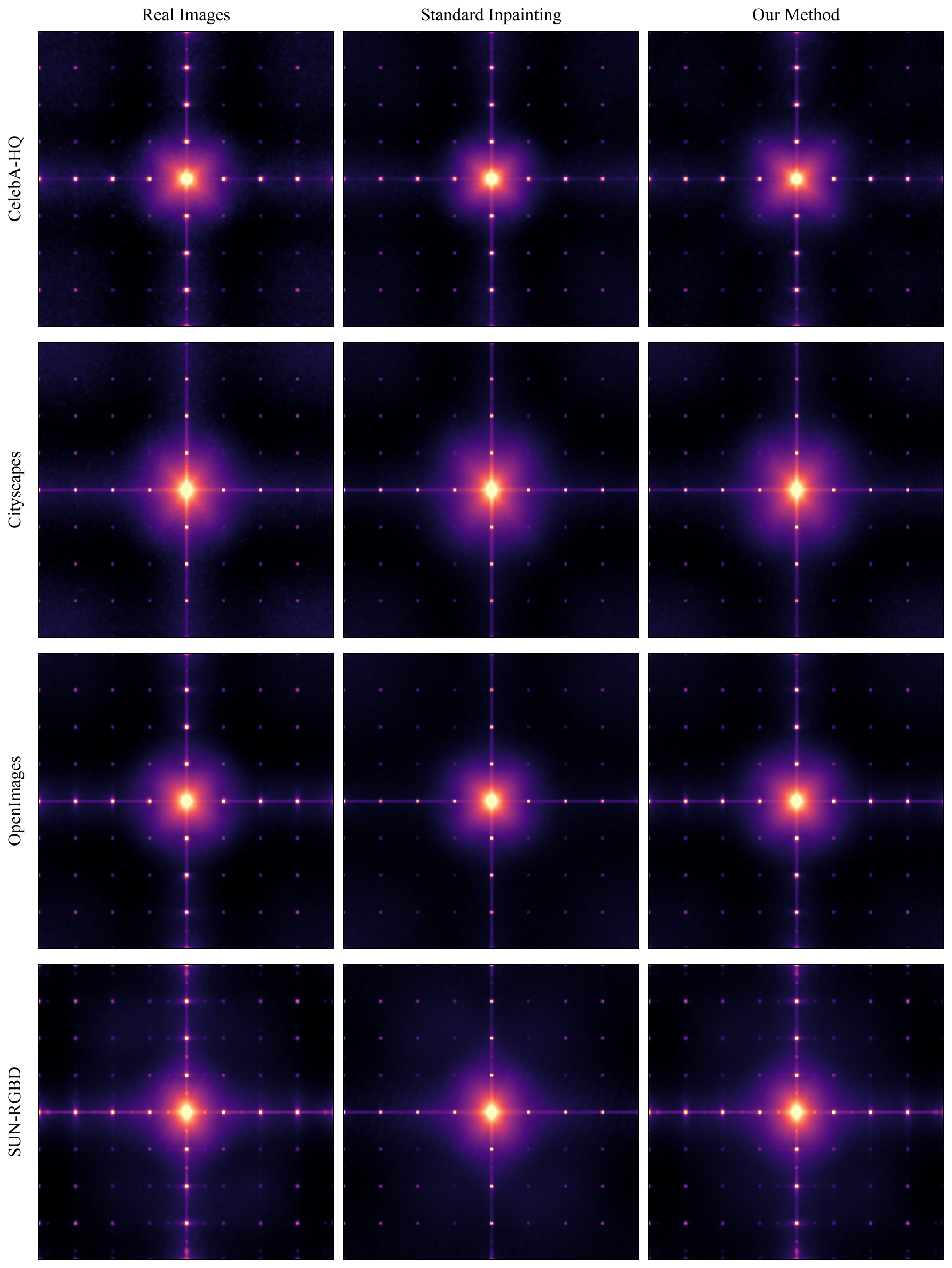}
    \caption{Frequency-domain analysis: standard inpainting introduces global spectral artifacts that INP-X eliminates.}
    \label{fig_fft_analysis}
\end{figure}

\begin{table}[h]
\centering
\caption{Spectral difference scores (MSE $\times 10^3$), real/fake.}
\label{tab_spectral_scores}
\vskip 0.05in
\begin{small}\begin{sc}
\begin{tabular}{lcc}
\toprule
Dataset & INP & \textbf{INP-X} \\
\midrule
CelebA-HQ & 15.31 & \textbf{12.87} \\
CityScapes & 11.47 & \textbf{5.89} \\
OpenImages & 17.13 & \textbf{1.92} \\
SUN-RGBD & 23.42 & \textbf{1.94} \\
\bottomrule
\end{tabular}
\end{sc}\end{small}
\end{table}

\subsection{Sub-Dataset Analysis}
\label{sec:supp_subdata}
To ensure our findings are consistent across diverse distributions, we analyze detector performance on sub-folders within each dataset. \cref{fig_app_subdata_bars} shows accuracy breakdowns by sub-folder, and \cref{fig_app_real_acc} reports accuracy on real images across datasets. The INP$\to$INP-X collapse is uniform across sub-distributions.

\begin{figure*}[h]
    \centering
    \includegraphics[width=0.42\textwidth]{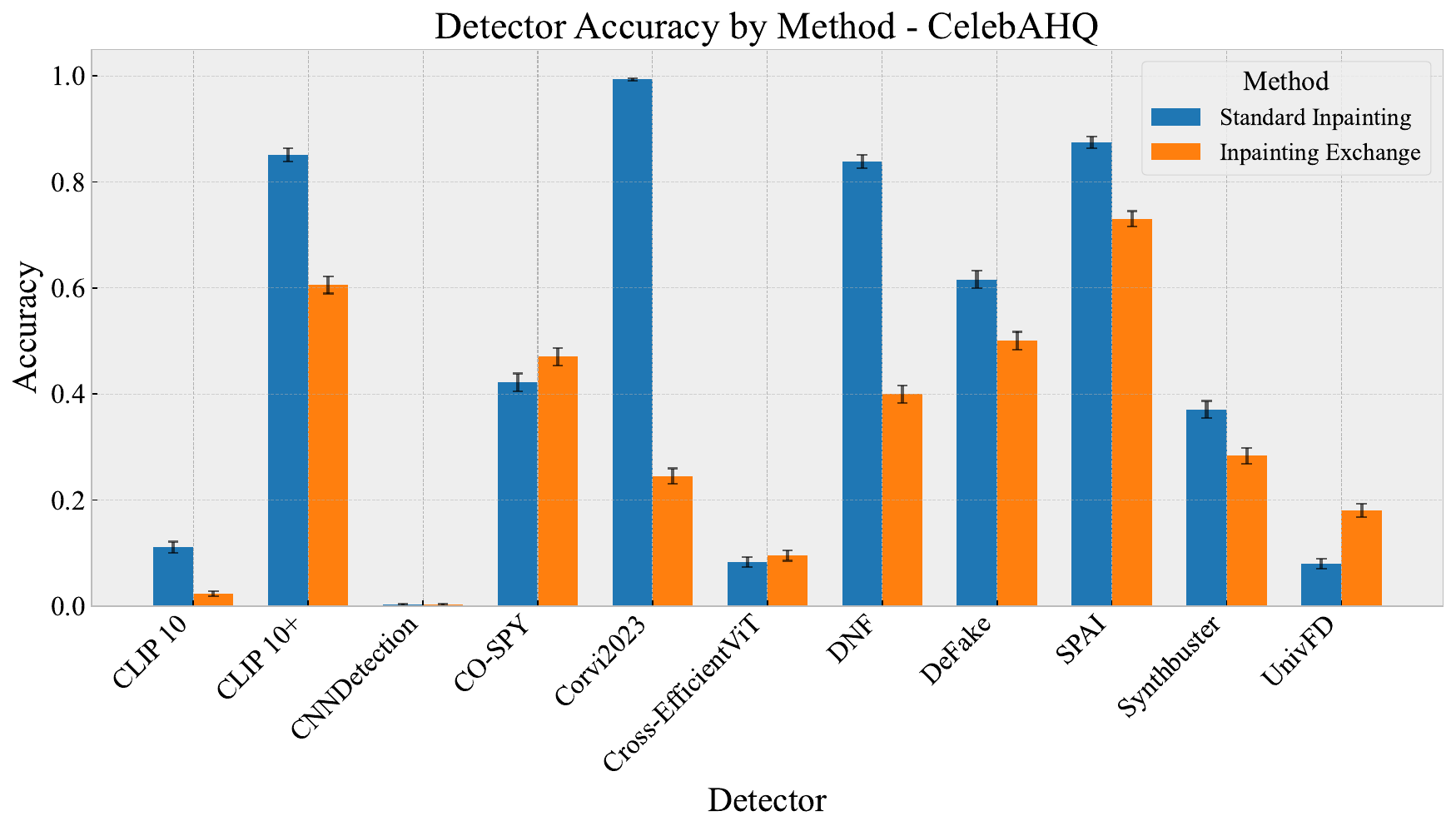}\hfill
    \includegraphics[width=0.42\textwidth]{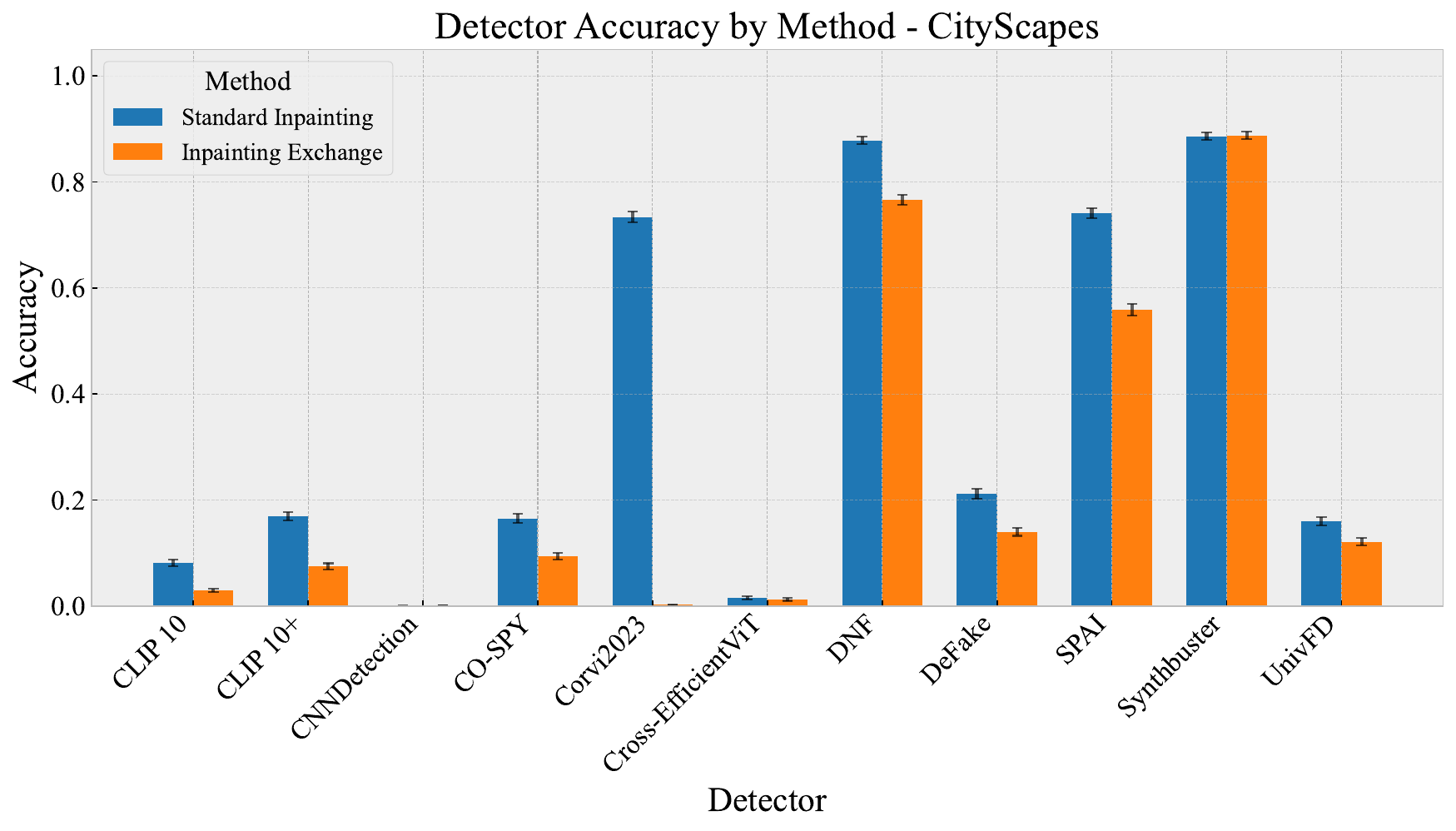}\\[2pt]
    \includegraphics[width=0.42\textwidth]{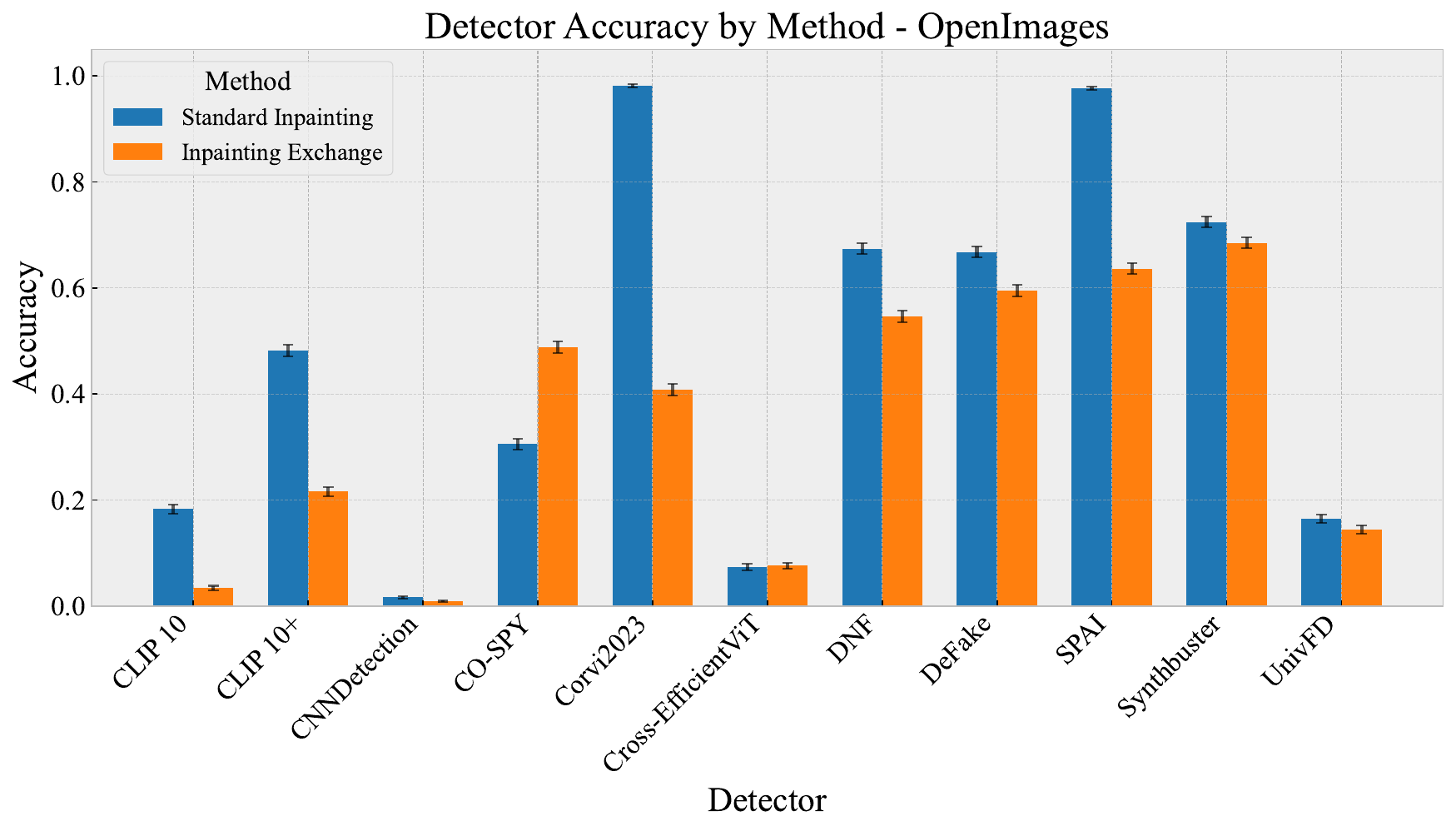}\hfill
    \includegraphics[width=0.42\textwidth]{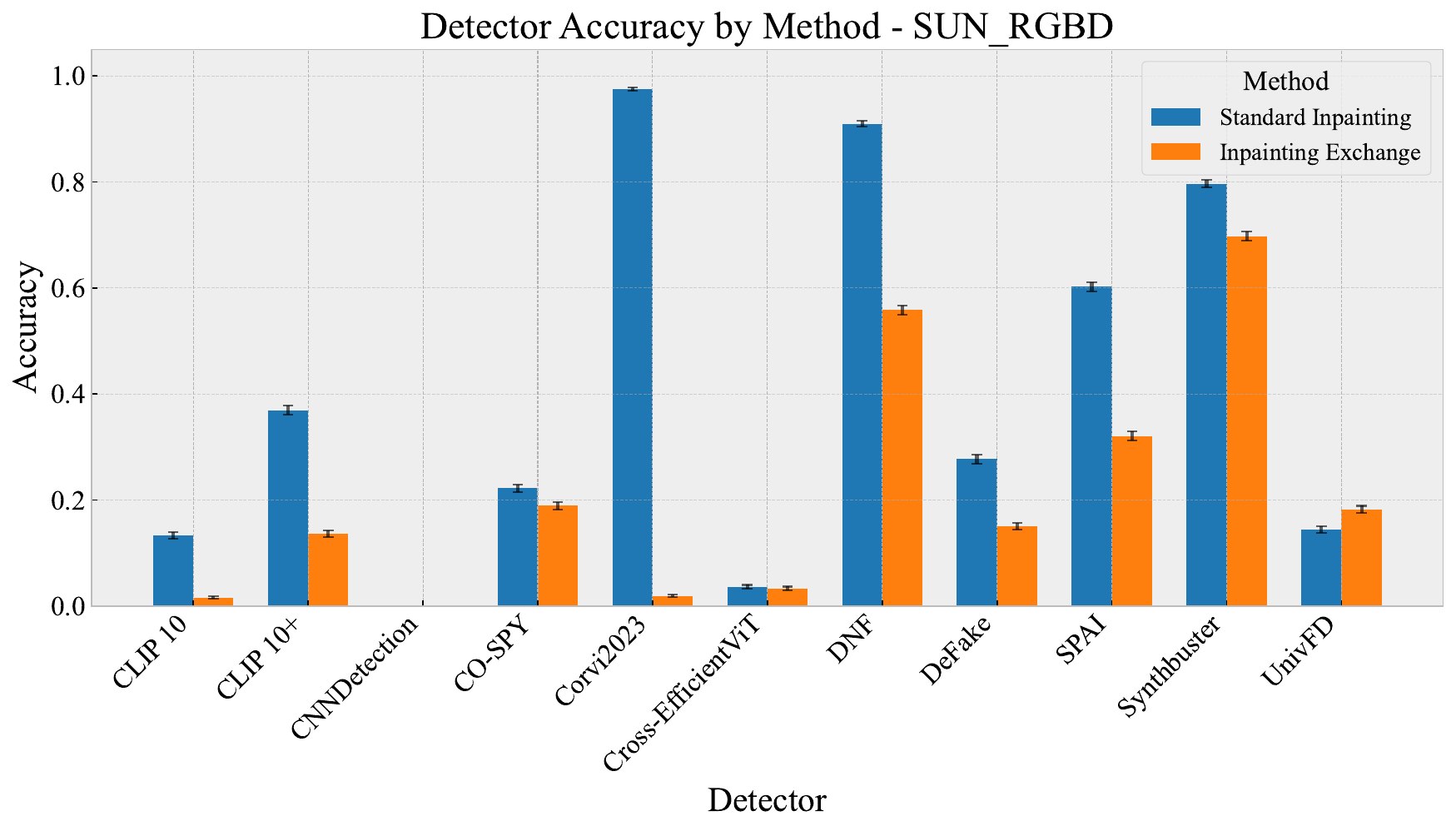}
    \caption{Accuracy breakdown by sub-folder for each dataset.}
    \label{fig_app_subdata_bars}
\end{figure*}

\begin{figure}[h]
    \centering
    \includegraphics[width=0.8\linewidth]{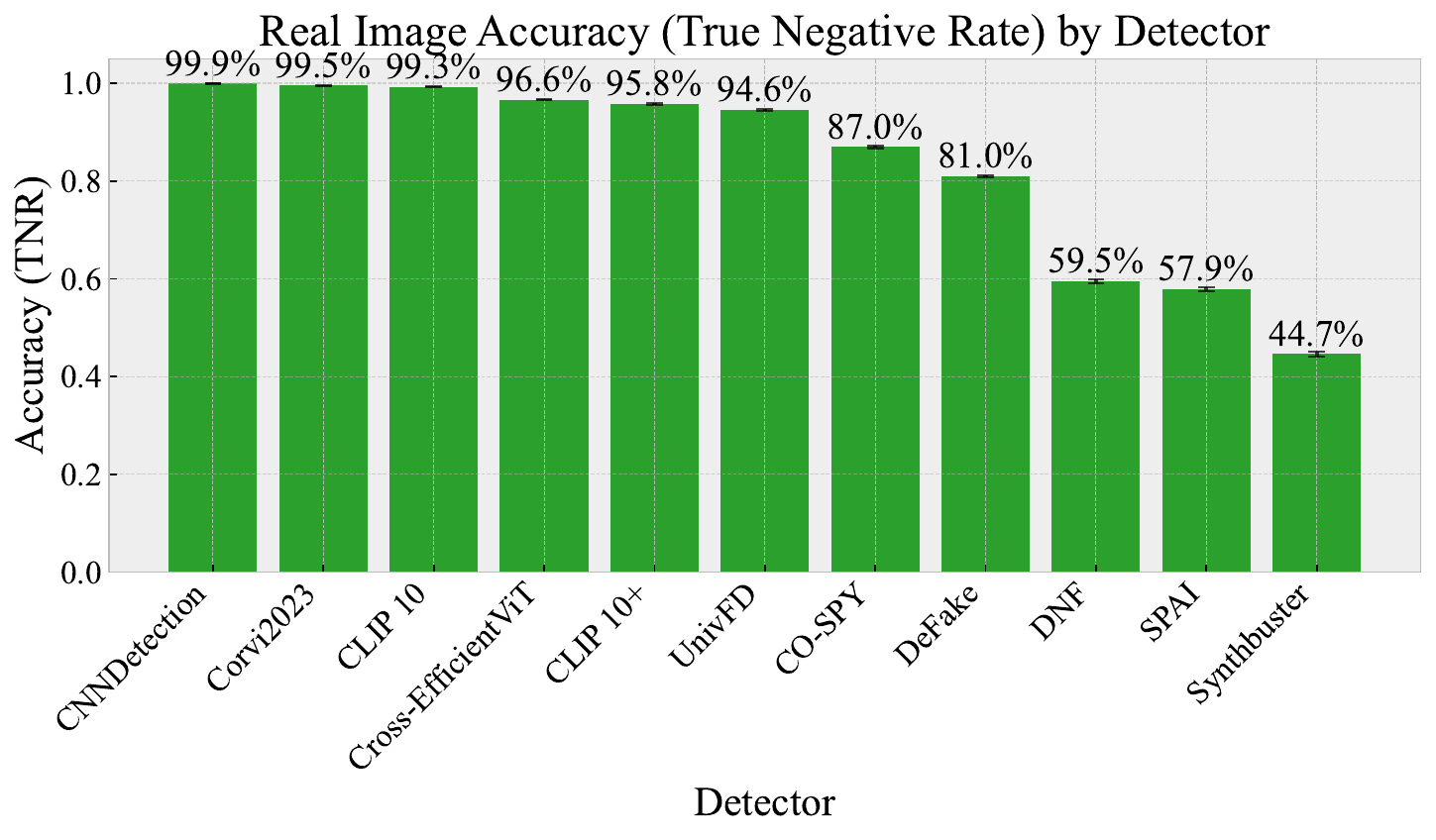}
    \caption{Detector accuracy on real images across datasets.}
    \label{fig_app_real_acc}
\end{figure}

\subsection{VAE Artifacts in Improved Models and Dose--Response}
\label{sec:supp_vae_advanced}
We extend the analysis to SDXL~\cite{podell2024sdxl} and FLUX.1~\cite{bfl2024flux}. Despite higher-channel latents, the VAE-loss$\leftrightarrow$high-frequency correlation persists (\cref{fig_vae_sdxl_flux}), confirming the artifact is a fundamental property of autoencoding-based generation. For \textbf{SDXL} (image-level Pearson $r$/Spearman $\rho$): VAE-loss vs.\ high-freq $r{=}0.937,\rho{=}0.901$; VAE-loss vs.\ inpaint-error $r{=}0.162,\rho{=}0.367$. Pixel-level (mean$\pm$std): VAE vs.\ inpaint $0.266{\pm}0.162$; VAE vs.\ high-freq $0.560{\pm}0.068$. For \textbf{FLUX.1} (16-channel VAE): VAE-loss vs.\ high-freq $r{=}0.936,\rho{=}0.840$; VAE-loss vs.\ inpaint-error $r{=}0.163,\rho{=}0.309$; pixel-level VAE vs.\ inpaint $0.211{\pm}0.136$, VAE vs.\ high-freq $0.450{\pm}0.072$.

Placing inpainters on a dose--response axis (\cref{tab_dose}) shows detectability is monotone in background contamination (Pearson $r{=}0.99$ over four pipelines): FLUX's $16$-channel latent leaks less than SD's $4$-channel latent (residual $1.63$ vs.\ $2.41$) and is correspondingly harder to detect on INP, exactly as \cref{prp_attenuation} predicts, while all pipelines converge to $\sim$$0.55$ under INP-X. This yields a falsifiable prediction for pixel-space transformers (\eg, PixelDiT): unless they composite explicitly, they should sit between FLUX and RePaint. \cref{tab_inpainter_collapse} confirms the collapse holds for every new inpainter across detectors.

\begin{figure*}[t]
    \centering
    \begin{subfigure}{0.9\textwidth}
        \includegraphics[width=\linewidth]{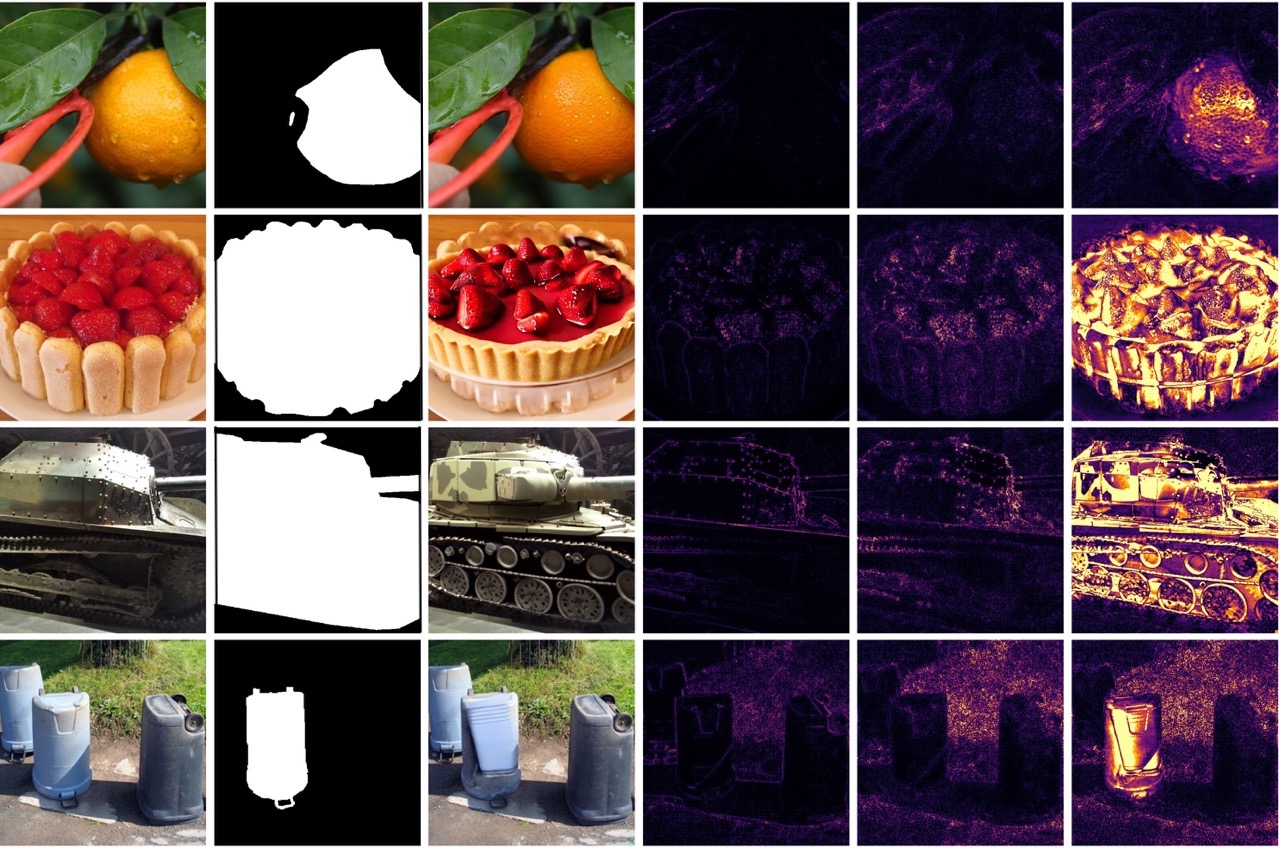}
        \caption{SDXL}
    \end{subfigure}
    \\[6pt]
    \begin{subfigure}{0.9\textwidth}
        \includegraphics[width=\linewidth]{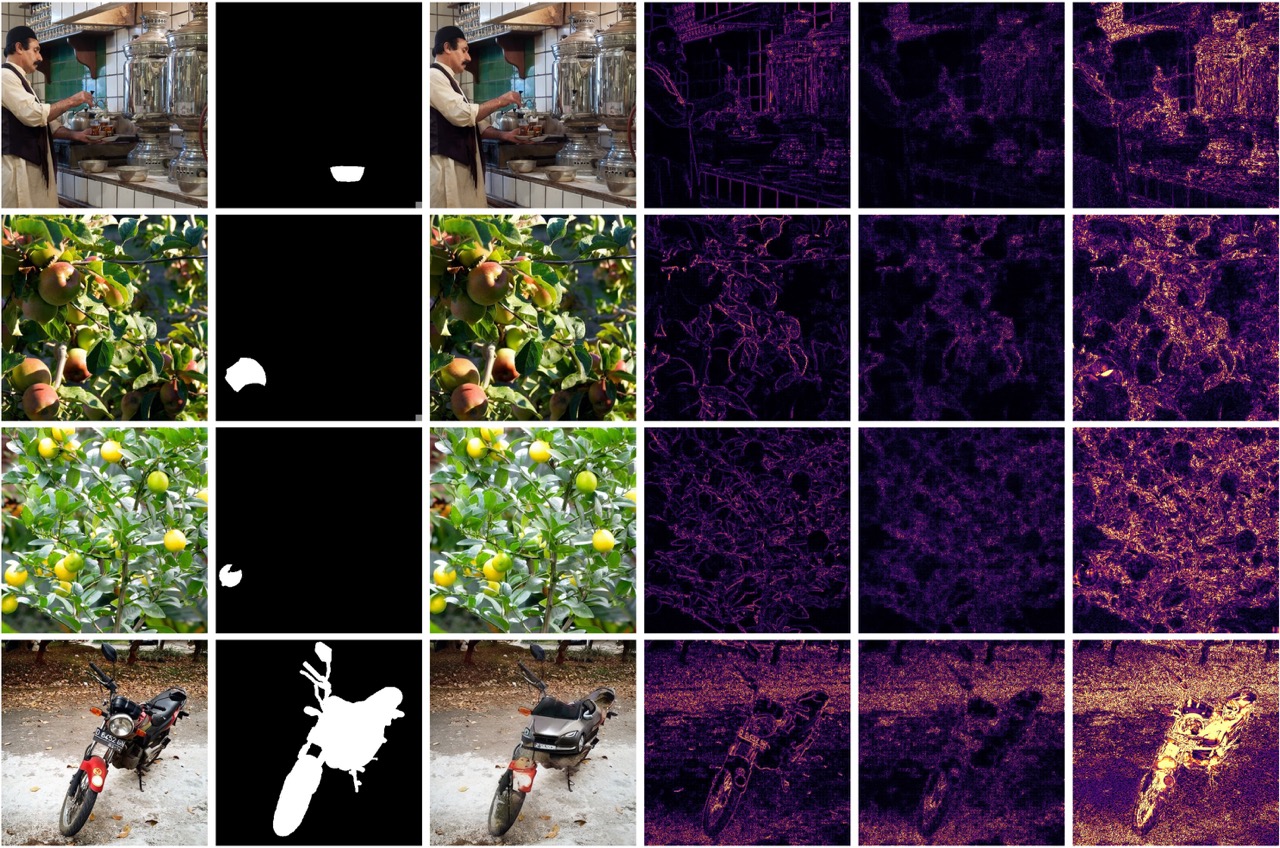}
        \caption{FLUX.1}
    \end{subfigure}
    \caption{Improved models still exhibit VAE reconstruction artifacts correlated with the inpainting error. Columns as in main \cref{fig_vae_verification}.}
    \label{fig_vae_sdxl_flux}
\end{figure*}

\begin{table}[h]
\centering
\caption{Dose--response: detectability is monotone in background latent residual (Pearson $r{=}0.99$); INP-X is flat.}
\label{tab_dose}
\vskip 0.05in
\begin{small}\begin{sc}
\resizebox{\columnwidth}{!}{%
\begin{tabular}{lcccc}
\toprule
Pipeline & Latent & BG resid. & Corvi INP & Corvi INP-X \\
\midrule
SD 1.4 & 4-ch & 2.41 & 0.951 & 0.549 \\
SDXL & 4-ch & 2.18 & 0.917 & 0.566 \\
FLUX.1 Fill & 16-ch & 1.63 & 0.882 & 0.571 \\
RePaint & pixel & 0.00 & 0.548 & 0.548 \\
\bottomrule
\end{tabular}%
}
\end{sc}\end{small}
\end{table}

\begin{table}[h]
\centering
\caption{INP$\to$INP-X collapse holds for modern inpainters across detectors.}
\label{tab_inpainter_collapse}
\vskip 0.05in
\begin{small}\begin{sc}
\setlength{\tabcolsep}{3pt}
\resizebox{\columnwidth}{!}{%
\begin{tabular}{lcccc}
\toprule
Inpainter & Corvi & DDA & B-Free & Hive \\
\midrule
SDXL & 0.917$\to$0.566 & 0.961$\to$0.593 & 0.968$\to$0.796 & 0.897$\to$0.559 \\
FLUX & 0.882$\to$0.571 & 0.943$\to$0.601 & 0.959$\to$0.782 & 0.874$\to$0.566 \\
\bottomrule
\end{tabular}%
}
\end{sc}\end{small}
\end{table}

\subsection{Multiresolution (Wavelet) Analysis}
\label{sec:supp_wavelet}
We present a wavelet-theoretic analysis of VAE-induced artifacts using the multiresolution framework of Mallat~\cite{mallat1989multiresolution,mallat1999wavelet}, giving a spatially localized counterpart to the Fourier-domain attenuation of \cref{prp_attenuation} and connecting the artifacts to classical rate--distortion limits.

\begin{definition}[Multiresolution approximation~\cite{mallat1989multiresolution}]
A multiresolution approximation of $L^2(\mathbb{R}^2)$ is a sequence of closed subspaces $\{V_j\}_{j\in\mathbb{Z}}$ with: (1)~$V_j\subset V_{j+1}$; (2)~$\bigcap_j V_j=\{0\}$ and $\overline{\bigcup_j V_j}=L^2$; (3)~$f(x)\in V_j\Leftrightarrow f(2x)\in V_{j+1}$; (4)~a scaling function $\phi\in V_0$ whose integer translates form an orthonormal basis of $V_0$. Define the wavelet subspace $W_j$ by $V_{j+1}=V_j\oplus W_j$, with orthogonal projector $\mathcal{P}_{W_j}$ capturing detail at scale $2^{-j}$.
\end{definition}
\vspace{-0.2em}
For any $x\in L^2(\mathbb{R}^2)$, $x=\sum_{j} \mathcal{P}_{W_j}x$ (convergence in $L^2$); for finite images the sum is truncated with standard boundary handling.

\begin{theorem}[Wavelet attenuation under spatial compression]
\label{thm_wavelet_attenuation}
Let $\mathcal{T}=\mathcal{D}\circ\mathcal{E}$ have latent spatial resolution reduced by $r=2^{j_c}$, and let $N_j$ be the number of scalar wavelet coefficients at scale $j$, with $N_{j_c}\approx HW/r^2$ and $N_j\approx N_{j_c}4^{\,j-j_c}$ (each finer scale has $4{\times}$ more coefficients). Assume: (1)~MSE distortion; (2)~an entropy constraint $H(Z)\le C$; (3)~a Gaussian rate--distortion model per scale in the low-rate regime $C/N_j\ll 1$. Then for any \emph{fine} scale $j>j_c$, with $K=\tfrac{2\ln 2\,C}{N_{j_c}}$ and a decoder-noise term $\sigma_j^2\ge 0$ that vanishes for the MMSE decoder $\tilde x=\mathbb{E}[x\mid z]$,
\begin{equation}
\mathbb{E}\!\left[\|\mathcal{P}_{W_j}\tilde{x}\|_2^2\right]
\le K\,4^{-(j-j_c)}\,\mathbb{E}\!\left[\|\mathcal{P}_{W_j}x\|_2^2\right]+\sigma_j^2.
\end{equation}
\end{theorem}

\begin{proof}[Sketch]
By the data-processing inequality, $I(\mathcal{P}_{W_j}X;\mathcal{P}_{W_j}\tilde X)\le I(X;Z)\le C$. Since $N_j\approx N_{j_c}4^{\,j-j_c}$, the per-coefficient rate is $R_j=C/N_j$ and the Gaussian rate--distortion distortion is $D_j=\sigma_{X,j}^2 2^{-2R_j}$ with $\sigma_{X,j}^2=\mathbb{E}[\|\mathcal{P}_{W_j}X\|_2^2]/N_j$. For small $R_j$, $1-2^{-2R_j}\approx 2\ln 2\,R_j$, so the retained per-coefficient energy scales as $\sigma_{X,j}^2\,2\ln 2\,C/N_j$; multiplying by $N_j$ gives the stated bound, with $\sigma_j^2$ absorbing decoder hallucination and non-ideal reconstruction.
\end{proof}
\vspace{-0.4em}
\begin{corollary}[Detectability via wavelet modulus maxima]
\label{cor:detectability}
For $\delta^{std}=\tilde x^{std}-x$ and fine scales $j>j_c$, $\mathbb{E}[\|\mathcal{P}_{W_j}\delta^{std}\|_2^2]\gtrsim \big(1-K\,4^{-(j-j_c)}\big)\mathbb{E}[\|\mathcal{P}_{W_j}x\|_2^2]$ (by orthogonality of the MMSE error) whenever residual exceeds the noise floor $\sigma_j^2$, so significant wavelet modulus maxima appear in $\Omega_{bg}$ and propagate across scales. Under INP-X, $\tilde x^{ex}|_{\Omega_{bg}}=x|_{\Omega_{bg}}$ forces $\mathcal{P}_{W_j}\delta^{ex}|_{\Omega_{bg}}=0$ for all $j$, confining maxima to the mask boundary $\partial M$ and foreground $\Omega_{fg}$.
\end{corollary}

\begin{remark}[Parseval equivalence]
For orthonormal wavelets, $\|x\|_2^2=\sum_j\|\mathcal{P}_{W_j}x\|_2^2$, equivalent (up to normalization) to Parseval's identity in the Fourier domain. The Fourier spectral gap of \cref{prp_attenuation} thus corresponds to attenuation of $\|\mathcal{P}_{W_j}\tilde x\|_2$ at fine scales $j>j_c$; the wavelet formulation makes the deficit spatially localized. Finite-image boundary handling affects only $O(H{+}W)$ coefficients and does not change the $4^{-(j-j_c)}$ decay.
\end{remark}
\vspace{-0.4em}
\subsection{Verifying the Dominance Condition ($\star$)}
\label{sec:supp_dominance}
\cref{thm_divergence} assumes the background divergence term dominates the foreground change, condition~(\ref{eq:dominance}). We verify it empirically by mask size: the strict inequality holds up to ${\sim}80\%$ mask area and becomes non-significant only in the degenerate near-full-mask regime, where INP-X coincides with INP. $99.0\%$ of the benchmark falls below this crossover, so the proposition (and \cref{cor_accuracy_ceiling}, which inherits the assumption) applies to essentially all images we study.
\vspace{-0.4em}
\subsection{Why Not Restore After Generation?}
\label{sec:supp_restore}
Our findings raise a question about diffusion-based inpainting pipelines: if the denoiser necessarily operates over the entire latent during generation, why not restore the original pixels outside the mask afterward? During generation, full-image denoising is required, since spatial locality is entangled in latent space and the decoder must reconstruct the complete image for smooth, coherent boundaries. Once generation is complete, however, there is no fundamental reason to keep the re-rendered background. INP-X does exactly this, replacing pixels where $M{=}0$ with the original image and eliminating the global artifact. Naive replacement could introduce seams, but these are removable by gradient-domain (Poisson) blending~\cite{perez2003poisson}, and \cref{sec:supp_edge} shows the accuracy drop is not caused by edges.

\subsection{Pixel-Space (VAE-less) Inpainting}
\label{sec:supp_vaeless}
Modern inpainting predominantly uses VAE-based latent diffusion for efficiency. Pixel-space methods such as RePaint~\cite{lugmayr2022repaint} avoid VAE reconstruction artifacts in the unmasked region (they operate directly in pixel space), but face practical limits:
\begin{itemize}[leftmargin=1.2em,itemsep=1pt]
    \item \textbf{Inference latency.} Pixel-space diffusion is very expensive: RePaint takes ${\approx}9$\,min per $256{\times}256$ image on a P100, versus $<$$30$\,s for latent models at higher resolution.
    \item \textbf{Generalization.} RePaint typically requires training separate models per dataset (\eg, a dedicated CelebA-HQ model), lacking the zero-shot generality of latent models.
    \item \textbf{Quality.} As shown in \cref{fig_repaint_artifacts}, while the background is preserved, inpainting quality can be inconsistent.
\end{itemize}
RePaint is our pixel-space anchor: it preserves the background natively and already sits at chance without exchange (\cref{tab_dose}).

\begin{figure}[h]
    \centering
    \includegraphics[width=0.8\linewidth]{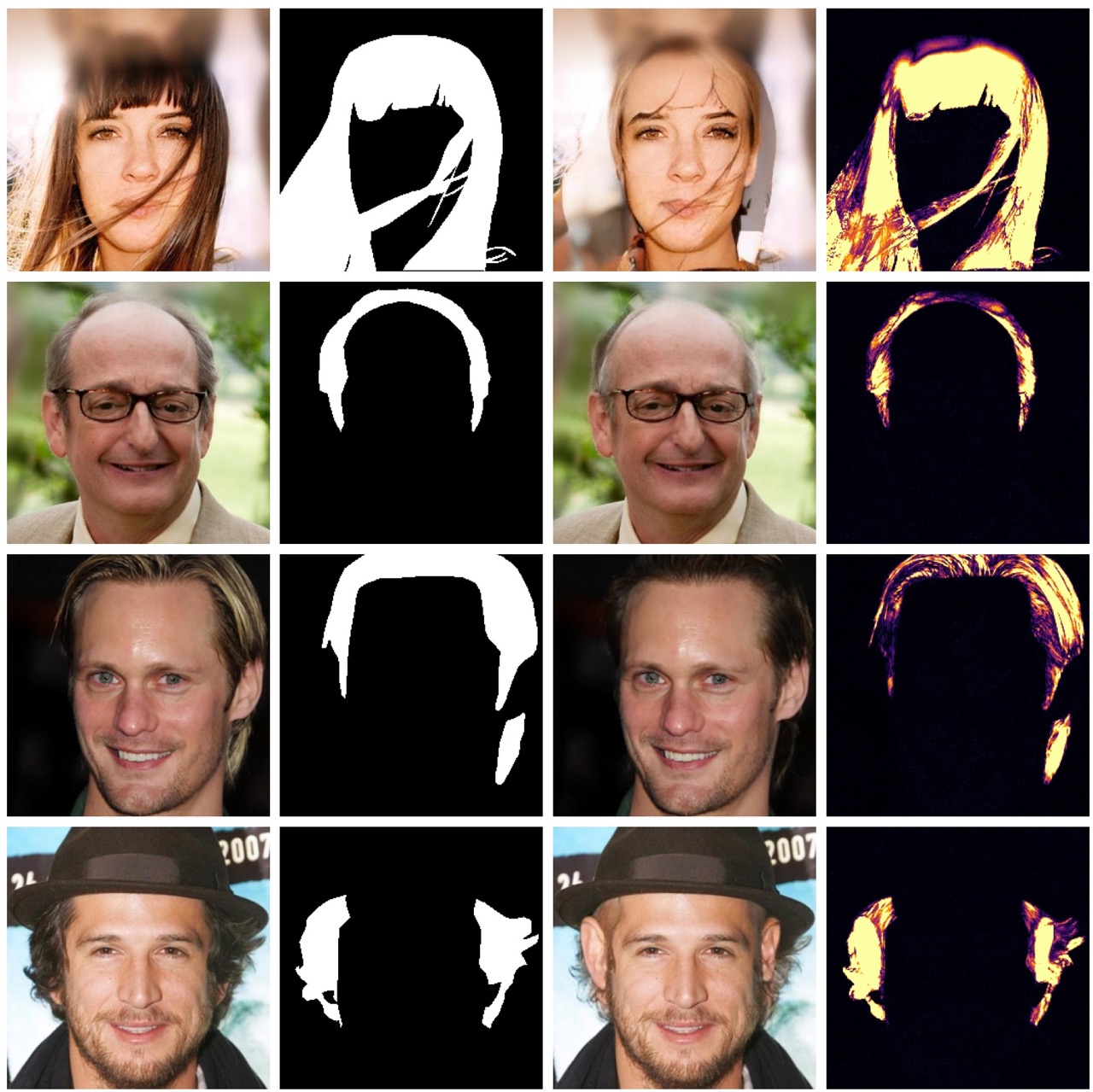}
    \caption{RePaint (VAE-less): the unmasked region is perfectly preserved (zero difference), but inference is prohibitively slow and needs dataset-specific training.}
    \label{fig_repaint_artifacts}
\end{figure}

\subsection{Additional Qualitative Results}
\label{sec:supp_qualitative}
\cref{fig_app_more_localizations} shows Grad-CAM localization: standard inpainting triggers global attention, while INP-X concentrates attention on the edited region. \cref{fig_more_differences} shows difference maps confirming INP-X differences are confined to the mask.

\begin{figure}[h]
    \centering
    \includegraphics[width=0.92\linewidth]{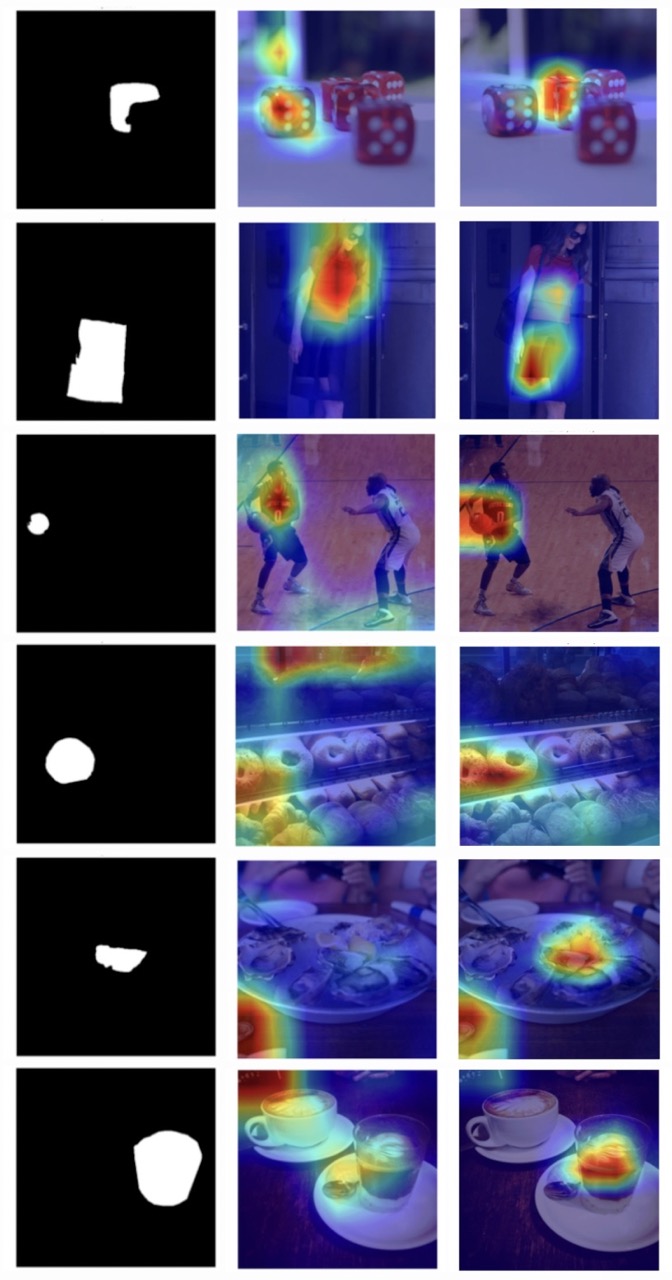}
    \caption{Detector attention (Grad-CAM). (a)~mask; (b)~INP triggers global attention across the image; (c)~INP-X localizes attention to the edited region.}
    \label{fig_app_more_localizations}
\end{figure}

\begin{figure*}[h]
    \centering
    \includegraphics[width=0.95\textwidth]{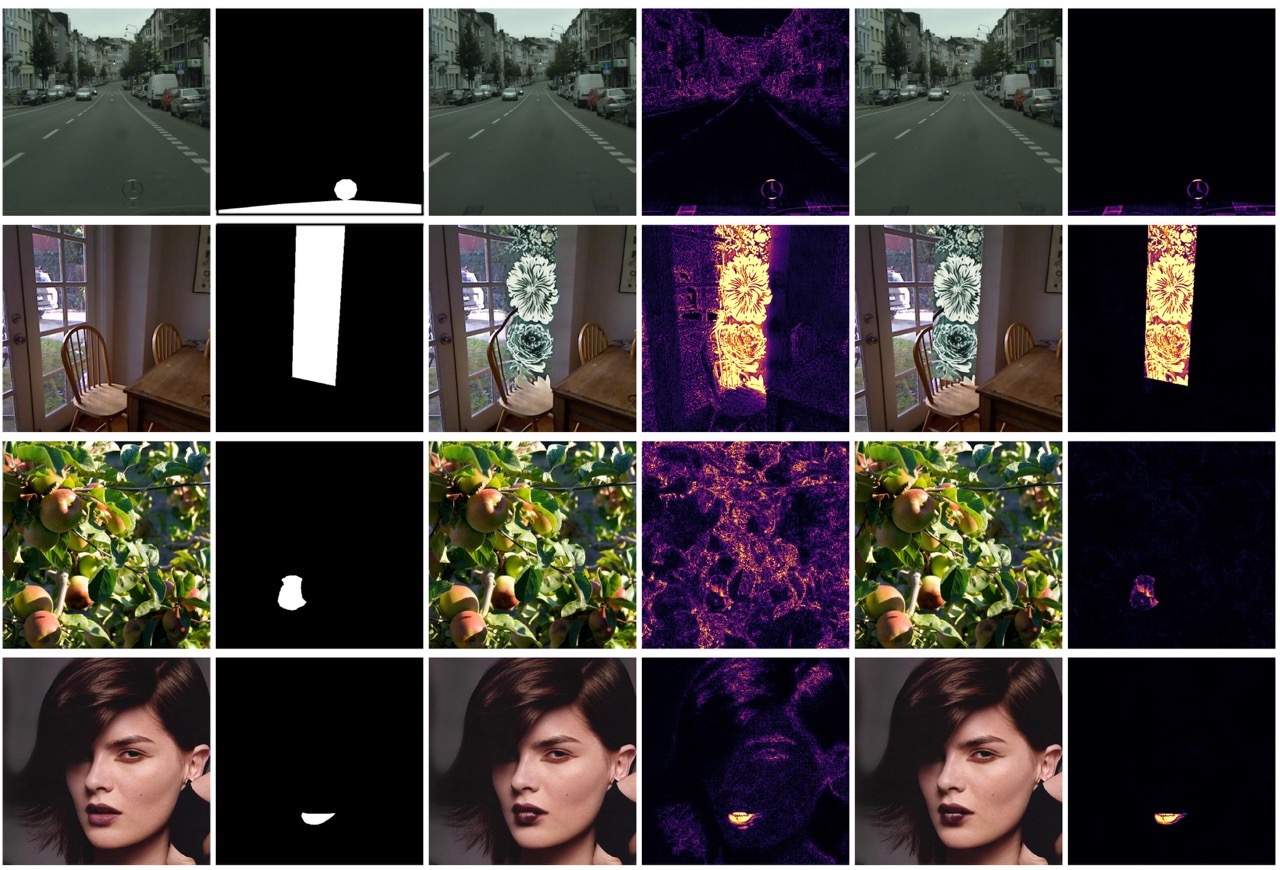}
    \caption{Standard inpainting vs.\ INP-X. Left to right: original, mask, INP, difference (Original$-$INP), INP-X, difference (Original$-$INP-X). INP differs across the whole image; INP-X differs only inside the mask.}
    \label{fig_more_differences}
\end{figure*}

\section{Human Perceptual Study}
\label{sec:supp_human}

\paragraph{Study design.}
We recruited $153$ participants across $5$ continents ($10$ countries), spanning education levels and self-reported AI expertise for broad demographic coverage. Each completed a binary forced-choice task: given a single image, participants judged whether it was real or AI-edited. Each did $20$ trials in randomized order with no feedback, no time limit, and no distractions, reflecting attentive, unpressured inspection. \cref{fig:study_instructions,fig:study_example_question} show the instructions and an example question.
\vspace{-1.0em}
\paragraph{Statistical analysis.}
Accuracy was computed per participant as the fraction of correctly classified images in each condition. Shapiro--Wilk tests rejected normality for all paired score differences, motivating non-parametric methods; conditions were compared with pairwise Wilcoxon signed-rank tests ($n{=}153$) with Bonferroni correction.
\vspace{-1.0em}
\paragraph{Results.}
Human accuracy was $71.6\%$ on real images, $63.8\%$ on INP, and $56.2\%$ on INP-X, all above chance ($p<0.001$). Pairwise: real vs.\ INP-X $\Delta{=}15.4$\,pp ($W{=}2446.5$, $p<0.001$); INP vs.\ INP-X $\Delta{=}7.6$\,pp ($W{=}2236.5$, $p{=}0.002$); real vs.\ INP $\Delta{=}7.8$\,pp ($W{=}3799.0$, $p{=}0.003$). A signal-detection analysis (pooled hit rates against the $28.4\%$ false-alarm rate on real images) gives $d'$ falling from $0.92$ (INP) to $0.73$ (INP-X): significantly reduced, not eliminated. The $7.6$-point INP$\to$INP-X gap was consistent across mask sizes (large masks $68.1\%$ vs.\ $59.4\%$; small masks $61.2\%$ vs.\ $54.3\%$).
\vspace{-1.0em}
\paragraph{The gap tracks out-of-mask salience.}
Why does human accuracy drop when the only nominal change is an imperceptible background artifact? Because the VAE re-decodes the \emph{whole} canvas, degrading salient content \emph{outside} the mask. \cref{tab_human_salience} restricts to trials whose mask does not overlap the salient region: the INP$\to$INP-X gain is $15.5$\,pt on CelebA-HQ faces but a non-significant $0.6$\,pt on faceless indoor scenes. The gap is monotone in out-of-mask salience and fully mediated by background LPIPS ($\rho{=}0.62$): restricting to the lowest-LPIPS quartile, INP falls to $57.4\%$, indistinguishable from INP-X. On CelebA-HQ trials with the mask off the face, the out-of-mask face is still altered (ArcFace cosine $0.912$ vs.\ $1.000$). Humans, like detectors, partly read the global artifact: the thesis reproduces in a second modality.

\begin{table}[h]
\centering
\caption{Human accuracy by out-of-mask salience. The INP$\to$INP-X drop tracks background salience, not the edit.}
\label{tab_human_salience}
\vskip 0.05in
\begin{small}\begin{sc}
\resizebox{\columnwidth}{!}{%
\begin{tabular}{llccc}
\toprule
Subset & Out-of-mask salience & INP & INP-X & $\Delta$ \\
\midrule
CelebA-HQ & faces & 70.6 & 55.1 & 15.5 \\
CityScapes & signage/text & 61.7 & 56.5 & 5.2 \\
SUN-RGBD & none & 57.9 & 57.3 & 0.6 \\
\bottomrule
\end{tabular}%
}
\end{sc}\end{small}
\end{table}

\begin{figure*}[h]
\centering
\begin{subfigure}[t]{0.48\textwidth}
    \includegraphics[width=\linewidth]{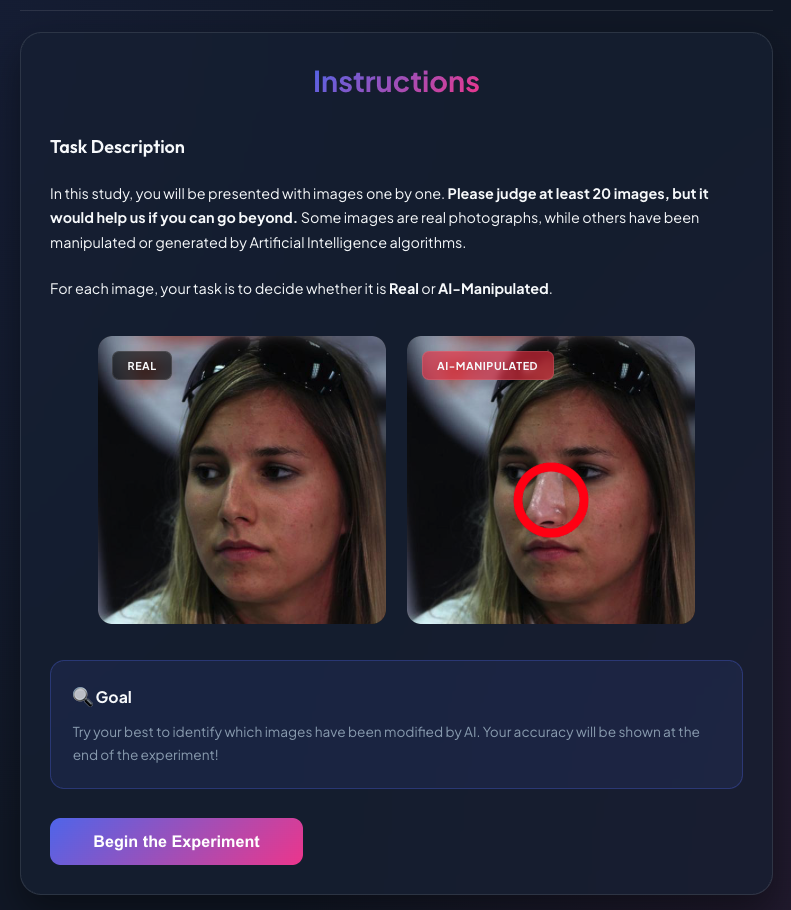}
    \caption{Instructions shown to participants at the start of the study.}
    \label{fig:study_instructions}
\end{subfigure}
\hfill
\begin{subfigure}[t]{0.48\textwidth}
    \includegraphics[width=\linewidth]{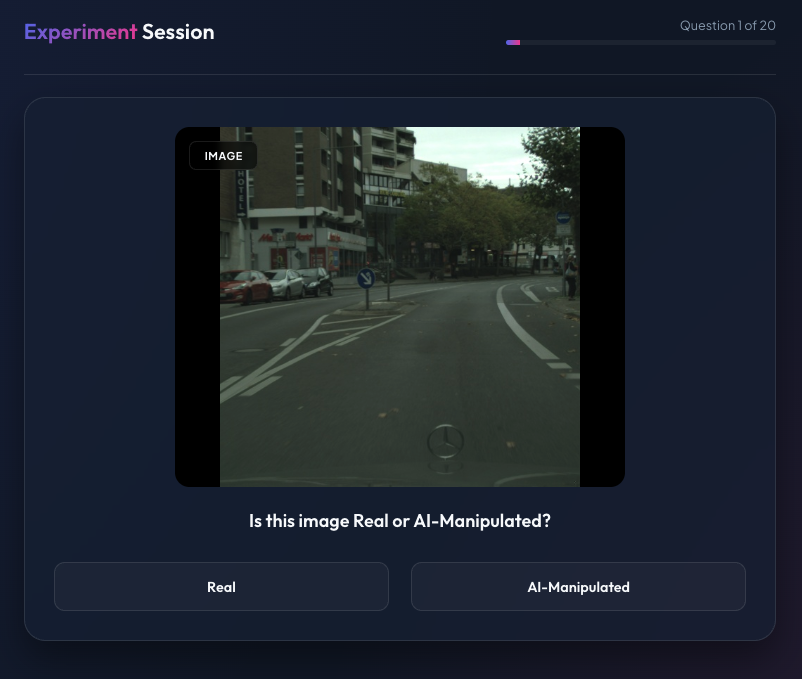}
    \caption{Example trial as presented to participants.}
    \label{fig:study_example_question}
\end{subfigure}
\caption{Human study interface: task instructions (left) and a representative trial (right).}
\label{fig_human_interface}
\end{figure*}

\end{document}